%% file: main.tex
\RequirePackage{fix-cm}
\documentclass[smallextended]{svjour3}       % onecolumn (second format)
\smartqed  % flush right qed marks, e.g. at end of proof
\usepackage{graphicx}
\usepackage{relsize}
\usepackage{graphicx}
\makeatletter
\providecommand{\caption@documentclass}{standard}
\makeatother
\usepackage{subcaption}
\usepackage{amsmath}
\usepackage{mathtools}
\usepackage{algorithm}
\usepackage[noend]{algpseudocode}
\usepackage{xcolor,colortbl}
\usepackage{amssymb}
\usepackage{moresize}
\usepackage{multirow}
\usepackage{mdframed}
\usepackage[hidelinks]{hyperref}
\usepackage{stmaryrd}
\usepackage{setspace}
\usepackage{enumitem}
\usepackage{booktabs}
\usepackage[round]{natbib}

\usepackage{adjustbox}
\usepackage{array}
\usepackage{makecell}

\usepackage{tikz}
    \usetikzlibrary{decorations.pathreplacing}
\tikzstyle{rect}=[rectangle,fill=white,draw]
\tikzstyle{simple}=[-,draw]
\tikzstyle{fleche}=[-latex,draw]
\tikzstyle{none}=[inner sep=0pt]

\usepackage{comment}

\usepackage{prettyref}
\newrefformat{def}{Definition~\ref{#1}}
\newrefformat{prop}{Proposition~\ref{#1}}
\newrefformat{th}{Theorem~\ref{#1}}
\newrefformat{lemma}{Lemma~\ref{#1}}
\newrefformat{algo}{Algorithm~\ref{#1}}
\newrefformat{fig}{Figure~\ref{#1}}
\newrefformat{tab}{Table~\ref{#1}}
\newrefformat{ex}{Example~\ref{#1}}
\newrefformat{sec}{Section~\ref{#1}}

\graphicspath{{./}}

\definecolor{gray50}{gray}{0.45}

\journalname{Machine Learning}

\begin{document}

\title{Can Transformers Learn to Verify During Backtracking Search?}
\titlerunning{Transformers Learning to Verify During Backtracking Search}

\author{Yin Jun Phua \and Tony Ribeiro \and Tuan Nguyen \and Katsumi Inoue}
\authorrunning{Y.J.~Phua et al.}

\institute{Yin Jun Phua (corresponding author) \at
  Institute of Science Tokyo, 2-12-1 Ookayama, Meguro-ku, Tokyo 152-8550, Japan \\
  \email{phua@comp.isct.ac.jp}
  \and
  Tony Ribeiro \at
  Centrale Nantes, CNRS, Laboratoire des Sciences du Num\'erique de Nantes,\\
  LS2N, UMR 6004, F-44000 Nantes, France \\
  National Institute of Informatics, 2-1-2 Hitotsubashi, Chiyoda-ku, Tokyo 101-8430, Japan \\
  Steelous Protocol, 8-20-32, Ginza, Chuo-ku, Tokyo 104-0061, Japan
  \and
  Tuan Nguyen \at
  Hanoi University of Science and Technology, No. 1 Dai Co Viet, Hai Ba Trung, Ha Noi, Vietnam
  \and
  Katsumi Inoue \at
  National Institute of Informatics, 2-1-2 Hitotsubashi, Chiyoda-ku, Tokyo 101-8430, Japan
}

\date{Received: date / Accepted: date}

\maketitle

\begin{abstract}
Backtracking search underlies classical constraint solvers, planners, and theorem provers.
Recent transformer-based reasoning systems explore search trees over their own intermediate steps.
A common training recipe fits an autoregressive next-token loss on offline solver traces.
The model's input at each step is a \emph{cumulative trace} of all prior decisions.
The optimal continue-or-backtrack predictor depends only on the current search state, since two trajectories reaching the same state admit the same viable continuations.
We show that decoder-only transformers trained on cumulative traces fail this requirement in two ways: the trace can scatter state features across many positions (\emph{scattered retrieval}), and the predictor can condition on the trajectory rather than the state (\emph{history entanglement}).
We address scattered retrieval with \emph{localization}, a trace-level fix that rewrites each decision block to expose state features locally.
We address history entanglement with \emph{Selective State Attention} (SSA), a fixed attention mask that enforces state-based decisions structurally without modifying training data, objective, or parameters.
We focus on reactive verification, after propagation has exposed a contradiction.
We test SSA on 3-SAT, graph coloring, Blocks World, and backtracking parsing.
On same-state pairs that differ only in prior history, SSA emits identical decisions while a cumulative-trained causal baseline does not.
Our contribution is a diagnostic of transformer behavior on serialized trajectory data, paired with a structural fix.
Pretrained language models that search over their own reasoning steps may face the same failure.
Our analysis opens up inference-time context clearing as a candidate way to apply the same isolation without retraining.

\keywords{Transformers \and Backtracking search \and Attention masking \and Selective State Attention \and History entanglement \and State isolation}
\end{abstract}

%///////////////
% CONTENT begin
%///////////////

\input{ext_introduction}

\input{ext_preliminary}

\input{ext_theory}

\input{ext_ssa_experiments}

\input{ext_related_work}

\input{ext_analysis}

\input{ext_conclusion}

%///////////////
% CONTENT end
%///////////////

\begin{acknowledgements}
    The authors thank Kotaro Okazaki (Steelous Protocol) for providing funding and sponsorship that enabled this international collaborative project, in particular supporting the travel and stay of co-authors in Japan during the development of this work.
    This study was carried out using the TSUBAME4.0 supercomputer at Institute of Science Tokyo.
    Large-language-model assistance (Claude from Anthropic, GPT from OpenAI, and Gemini from Google) was used during manuscript preparation to improve the English of some passages; all technical content, analysis, and conclusions are the authors' own.
    Funding information is provided under Statements and Declarations.
\end{acknowledgements}

\section*{Statements and Declarations}

\paragraph{Funding.} This research has been funded by the European Union. Views and opinions expressed are however those of the authors only and do not necessarily reflect those of the European Union or ERCEA. Neither the European Union nor the granting authority can be held responsible for them. This work has also been supported by JSPS KAKENHI Grant Number JP25K03190 and JST CREST Grant Number JPMJCR22D3, Japan, and by ROIS NII Open Collaborative Research 2024-24S1203.

\paragraph{Competing Interests.} The authors have no competing interests to declare relevant to this article.

\paragraph{Ethics Approval.} Not applicable. This research did not involve human participants or animals.

\paragraph{Consent to Participate.} Not applicable.

\paragraph{Consent for Publication.} Not applicable.

\paragraph{Data Availability.} No third-party datasets are used in this study. All experimental traces are generated deterministically from the released code given the instance-generator parameters and random seeds documented in Appendix~\ref{sec:appendix-architecture}.

\paragraph{Code Availability.} The code supporting this study, including the SSA and baseline architectures, trace generation scripts, training and evaluation scripts, and table and figure generation scripts, is available at \url{https://github.com/phuayj/ssa-transformer} under an MIT license.

\paragraph{Author Contributions.} All authors (Y.J.P., T.R., T.N., K.I.) contributed to the study conception and design. Material preparation, implementation, and experimental analysis were performed by all authors. The first draft of the manuscript was written collectively, and all authors commented on previous versions. All authors read and approved the final manuscript.

\bibliographystyle{spbasic}      % basic style, author-year citations
\bibliography{Bibliography}   % name your BibTeX data base

\input{ext_appendix}

\end{document}

%% file: ext_introduction.tex
\section{Introduction} \label{sec:introduction}

Backtracking is a basic algorithmic primitive: a search extends a partial solution step by step, and when the current path fails to complete it retracts a decision and explores an alternative.
Classical constraint solvers, planners, and theorem provers implement this pattern in symbolic infrastructure that performs constraint propagation~\citep{davis1962machine}, viability checking, and conflict-driven backjumping~\citep{prosser1993hybrid,marques-silva2021conflict}.
Recent reasoning systems built on the transformer architecture~\citep{vaswani2017attention} implement the same pattern through the language model's autoregressive output, with the model proposing the next step and revising or backtracking when its trajectory dead-ends~\citep{yao2023tree,yang2025selfbacktracking,kim2025astro}.
Viewed through \citet{kowalski1979algorithm}'s slogan \emph{Algorithm = Logic + Control}, both regimes separate a logic side (maintaining constraints and exposing contradictions in partial states) from a control side (choosing which branch to extend and when to retract).
The decomposition was developed for classical symbolic programs but still frames transformer-based reasoning: a symbolic environment can supply the logic, while the transformer is asked to learn part of the control.

A recurring challenge across learned search is the evaluation of non-terminal states.
Process reward models score local step quality~\citep{lightman2023lets}; value models estimate future solve probability from a partial state~\citep{snell2024scaling}; dead-end detectors in classical planning predict whether a state has any feasible continuation~\citep{steinmetz2016towards}; outcome-only methods evaluate completed candidates instead~\citep{cobbe2021training,wang2023selfconsistency}; and neural-symbolic hybrids such as AlphaProof, AlphaGeometry, and FunSearch pair a neural proposer with a verifying solver that checks each candidate exactly~\citep{hubert2025alphaproof,trinh2024solving,romeraparedes2023mathematical}.
An alternative threads the verification decision back through the same neural model that proposes branches: the transformer is trained to internalize a continue-or-backtrack signal under the same next-token loss that drives its branching policy~\citep{velickovic2022clrs,nye2022show,pan2025backtracking}, with the model's input at each step accumulating every prior decision into a growing record of the search trajectory.
We study what predictor this last regime selects.

Whether transformers can learn algorithmic procedures from supervision at all has itself been an active question.
Transformers learn simple function classes in-context~\citep{garg2022what}, execute multi-step algorithms from chain-of-thought traces~\citep{wei2022chain,zhou2022teaching}, and acquire symbolic mappings~\citep{charton2022linear}.
Standard concerns are insufficient computation depth---addressed by length generalization~\citep{anil2022exploring,jelassi2023length} and looped architectures~\citep{lee2024teaching,giannou2023looped,yang2024looped}---and exposure-distribution shift in autoregressive imitation learning~\citep{ross2011reduction}.
Search adds a separate concern: the predictor selected by next-token loss on cumulative traces can condition on the path used to reach the current state rather than the state itself, even when the optimal target is state-only.

We formalize this concern as a structural requirement on the learned verifier.
Two trajectories that reach the same propagated state admit the same set of viable continuations, so the optimal continue-or-backtrack decision is a function of the state alone---a requirement we call \emph{state-equivalence}.
The target is state-local in a strong sense: a small multilayer perceptron (MLP) on per-variable state features solves the 3-SAT setting we study at near-ceiling solve rate.
We find that decoder-only transformers trained on cumulative solver traces fail to recover this state-local rule: the next-token training objective admits any predictor that fits the joint distribution of states and histories, including ones that exploit residual trajectory features that fail to generalize across same-state pairs.
We name the failure \emph{history entanglement}; it resembles compounding error in imitation learning but is partly representational, the model internalizing a path-dependent surrogate rather than a state-based rule.

We propose \emph{Selective State Attention} (SSA), a fixed attention mask~\citep{weston2023system,ho2024block} that blocks each decision step from attending to prior decision steps while preserving access to the problem prefix (Figure~\ref{fig:overview}).
SSA enforces history invariance structurally: with block-relative positional encoding, the logits at positions inside the current decision block depend only on the problem prefix and the current decision block, without parameter or objective changes.

\begin{figure}[t]
\centering
\includegraphics[width=\linewidth]{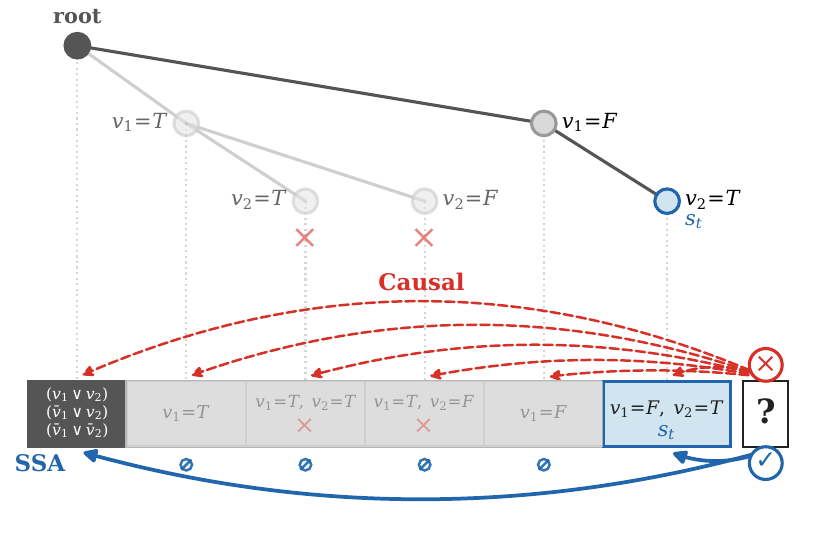}
\caption{SSA versus causal attention on a 3-SAT instance with decision blocks that record the full partial assignment. Red dashed arrows (causal) and blue solid arrows (SSA) show the attention scope at the current step.}
\label{fig:overview}
\end{figure}

\paragraph{Setting and scope.}
We delegate the logical component to symbolic infrastructure and ask whether a transformer can learn the control component.
The reactive form, where propagation has already exposed an empty domain or violated clause, is our focus; the proactive form (predicting locally consistent dead ends before any conflict surfaces) requires lookahead and is strictly harder, lying outside our scope.
The supervised learning problem runs over offline solver records; we call the growing input record at each step a \emph{cumulative trace}, the default training format for any autoregressive model whose context accumulates a record of its own past actions.

\paragraph{Contributions.}
\begin{enumerate}[nosep,leftmargin=*]
    \item \emph{Diagnosis.} We identify \emph{history entanglement} as a same-state invariance failure of causal transformers trained on cumulative traces, distinct from distribution shift in imitation learning (\S\ref{sec:history-transplant}).
    \item \emph{Method.} We propose Selective State Attention, a fixed attention mask with an exact prior-history invariance property under block-relative positions (Proposition~\ref{prop:ssa-invariance}), implemented without parameter or objective changes (\S\ref{sec:ssa-architecture}).
    \item \emph{Evidence.} Across state-rebuilt transfer tests, direct invariance diagnostics, and a sweep of alternative interventions on a verifier-only benchmark, mask-level isolation is the only intervention we test that restores reactive state-based verification from cumulative traces; the result holds in four structured search domains (\S\ref{sec:ssa-experiments}).
\end{enumerate}

\paragraph{Relation to the conference version.}
This paper is a substantially extended version of an earlier conference paper~\citep{ribeiro2025transformers}.

%% file: ext_preliminary.tex
\section{Background and Experimental Setup}
\label{sec:preliminary}

This section gives a formal definition of the learning problem (\S\ref{sec:problem-setting}), introduces tree traversal as a simpler instance that motivates the localization fix (\S\ref{sec:tree-traversal-recap}), and fixes the shared experimental setup (\S\ref{sec:slot-memory-method}).

\subsection{Problem Setting}
\label{sec:problem-setting}

\begin{definition}[Search step, action factorization, and reactive verification]
\label{def:problem-setting}
Let $X$ be a backtracking-solvable problem instance, e.g., a constraint satisfaction problem (CSP)~\citep{rossi2006handbook} such as 3-SAT or graph coloring, a planning problem such as Blocks World, or parsing expression grammar (PEG;~\citealp{ford2004parsing}) parsing. The search algorithm is a depth-first solver with fixed variable and value heuristics, running unit propagation between assignments and conflict-driven backjumping when a clause is falsified. Each step $t$ of the solver on $X$ produces
\begin{itemize}[nosep]
    \item a \emph{search state} $S_t$ comprising the partial assignment, propagated variable domains, and tried-alternative set;
    \item an \emph{action} $a_t \in \{\textsc{branch}(x, v),\, \textsc{backtrack}\}$ that factors into a \emph{policy} component (the choice of variable $x$ and value $v$ inside $\textsc{branch}(\cdot,\cdot)$) and a \emph{verification} component represented by the binary indicator
    \begin{equation}
    \label{eq:Y-definition}
        Y_t \;:=\; \mathbf{1}[a_t = \textsc{backtrack}] \in \{0, 1\};
    \end{equation}
    \item a token-serialized \emph{decision block} $B_t$ encoding $(S_t, a_t)$ (concrete examples in Appendix~\ref{sec:appendix-tokenization}).
\end{itemize}
$Y_t$ is the binary target of the verification analysis throughout the paper. Let $P$ token-serialize $X$. The \emph{cumulative trace} at step $t$ is the concatenation $(P, B_1, \ldots, B_t)$, and we write $H_{<t} = (B_1, \ldots, B_{t-1})$ for the prior history. The supervised learning task predicts $a_t$ from the cumulative trace at step $t-1$ extended by the state portion of $B_t$, under standard causal (autoregressive) attention. The Bayes-optimal predictor of $Y_t$ depends on $S_t$ alone, since two trajectories that reach the same state admit the same set of viable continuations.

The \emph{reactive verification target} is the state-local rule
\begin{equation}
\label{eq:reactive-rule}
Y_t^{\,*}(S_t) \;=\;
\begin{cases}
1 & \text{if propagation exposes a contradiction in } S_t,\\
0 & \text{otherwise,}
\end{cases}
\end{equation}
where an exposed contradiction is an empty domain or violated clause that constraint propagation produces. A state $S$ is \emph{viable} if it has at least one satisfying descendant in the search tree and \emph{dead} otherwise; the optimal action backtracks on dead states and continues on viable states. The search infrastructure handles states with no remaining untried alternative, so the model never sees them. The \emph{proactive} form additionally sets $Y = 1$ on locally consistent dead ends, where no contradiction has surfaced but no satisfying descendant exists; it requires lookahead and is strictly harder; the main contribution targets the reactive form. We drop the time subscript when context allows, writing $S$ for the search state, $H$ for the prior history, and $Y$ for the verification indicator.
\end{definition}

\begin{example}[3-SAT solving with backtracking]
\label{ex:sat-trace}
Consider the following 3-SAT instance on three variables $v_1, v_2, v_3$:
\begin{equation*}
C_1 = (v_1 \lor v_2 \lor v_3),\quad
C_2 = (\bar v_1 \lor \bar v_2),\quad
C_3 = (\bar v_1 \lor v_2 \lor \bar v_3),\quad
C_4 = (v_1 \lor \bar v_2 \lor \bar v_3).
\end{equation*}
A depth-first search produces the following trace, with each step labelled by its decision level $\ell$ and propagation result; an overlined literal $\bar v$ denotes negation.
\begin{center}
\footnotesize
\setlength{\tabcolsep}{4pt}
\captionsetup{type=table}
\caption{Backtracking-search trace for the 3-SAT instance of Example~\ref{ex:sat-trace}.}
\label{tab:sat-trace}
\begin{tabular}{@{}r l l l c@{}}
\toprule
$t$ & decision ($\ell$) & propagation & state $S_t$ & $Y_t$ \\
\midrule
$1$ & branch $v_1{=}T$ ($\ell{=}1$) & $C_2$ unit: $v_2{=}F$ & $\{v_1{=}T,\,v_2{=}F\}$ & $0$ \\
$2$ & branch $v_3{=}T$ ($\ell{=}2$) & $C_3$ falsified & conflict on $C_3$ & $1$ \\
$3$ & retry $v_3{=}F$ ($\ell{=}2$) & all clauses satisfied & $\{v_1{=}T,v_2{=}F,v_3{=}F\}$ & $0$ \\
\bottomrule
\end{tabular}
\end{center}
At step $1$ the policy branches $v_1{=}T$ and unit propagation on $C_2$ forces $v_2{=}F$. At step $2$ the policy branches $v_3{=}T$, which falsifies $C_3$; the verification target is $Y_2 = 1$ because $S_2$ contains an exposed contradiction. The search infrastructure then retries $v_3{=}F$ at step $3$, satisfying every clause. The transformer is responsible for emitting branch and backtrack actions; unit propagation, conflict exposure, and the level update remain external.

The trace above illustrates \emph{reactive} backtracking: the search returns to the deepest unexplored branch one level at a time. Conflict-driven clause learning (CDCL) augments this with conflict analysis: when the falsified clause traces the conflict back to a decision at an earlier level, the same infrastructure can backjump past intervening levels and resume search from the responsible decision~\citep{marques1996grasp,marques-silva2021conflict}. For example, a conflict that depends only on the level-$1$ assignment of $v_1$ would trigger a backjump from level $\ell{=}2$ all the way to level $\ell{=}1$, skipping any intermediate level-$2$ alternatives. The verification target $Y_t$ is unchanged under either policy: $Y_t = 1$ whenever $S_t$ contains an exposed contradiction, regardless of whether the search infrastructure resolves the conflict by retrying at the same level or by backjumping to an earlier one. The neural model therefore emits the same branch and backtrack tokens in both cases; only the infrastructure-side execution differs. Appendix~\ref{sec:appendix-tokenization} shows how a complete trace of this form is serialized into decision blocks $B_t$.
\end{example}

At deployment the model autoregressively emits action tokens, the infrastructure runs propagation between actions, and we count an instance as \emph{solved} once the solver reaches a satisfying assignment within a fixed token budget. Our benchmarks serve as diagnostics rather than solver-competitive comparisons: we ask whether a neural model matches the same search infrastructure under deployment shift, not whether it solves intrinsically hard instances.

\subsection{Tree Traversal and Localization}
\label{sec:tree-traversal-recap}

We begin with tree traversal, a simpler instance of the learning problem above with no constraint propagation and an exhaustive policy.
A decoder-only transformer reads a tree serialized as a parenthesized node list and autoregressively emits the visit order; the verification target reduces to deciding when to return to a parent node, which under depth-first search requires the transformer to recognize that it has visited all $k$ children.
We observe an asymmetry between traversal orders in this setting.
Breadth-first search (BFS) succeeds reliably under standard causal attention because the next visit follows from the most recent token in a queue serialized into the trace.
Depth-first search (DFS) fails because the natural encoding distributes the visited-children record across the trajectory rather than localizing it in the current decision context (Figure~\ref{fig:trace-encodings}).

\begin{figure}[t]
\centering
\includegraphics[width=\linewidth]{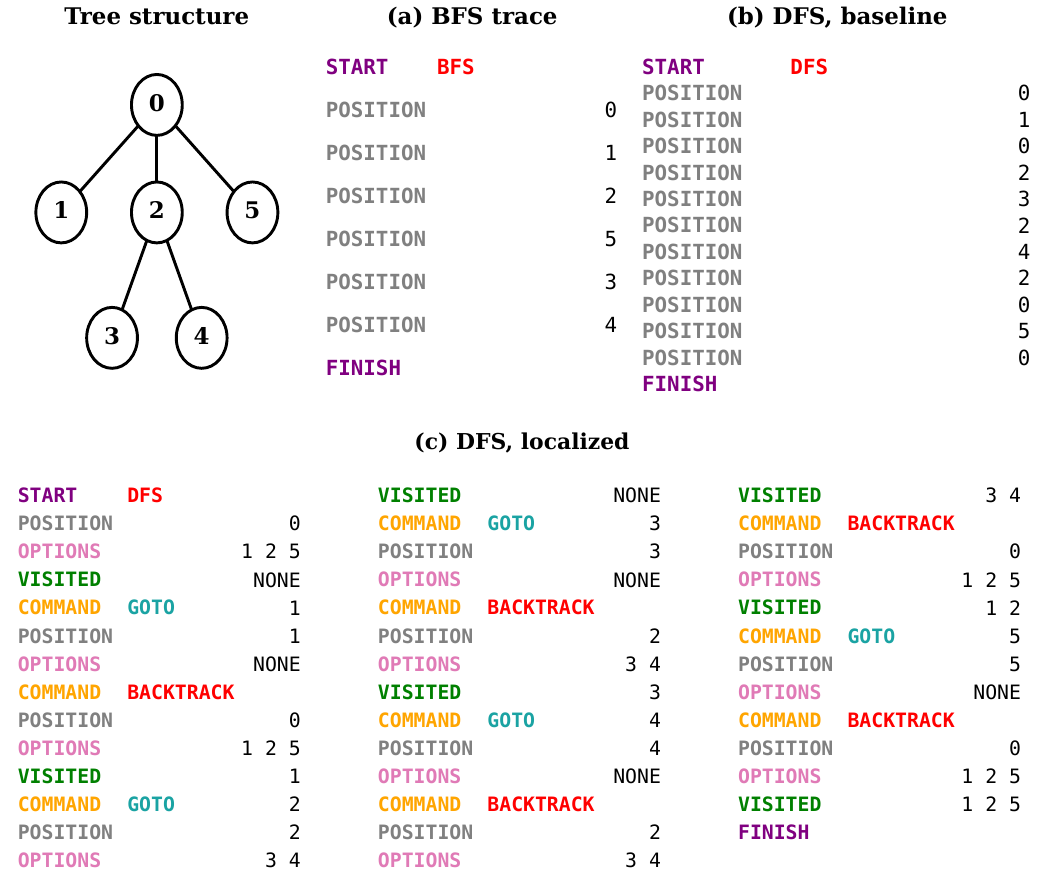}
\caption{Three trace encodings for tree traversal: \textbf{(a)} BFS baseline, \textbf{(b)} DFS under the natural encoding, and \textbf{(c)} \emph{localized} DFS in which the current context lists the already-visited children explicitly.}
\label{fig:trace-encodings}
\end{figure}

\paragraph{Scattered retrieval and localization.}
Verifying that all $k$ children have been visited under the natural DFS encoding requires the transformer to recombine $k$ pieces of information distributed across many trajectory positions, and per-token attention retrieval has bounded accuracy that compounds in the feature count.
We call this failure mode \emph{scattered retrieval}: the relevant state features are present in the trace but inaccessible to the local decision block.
The fix rewrites the trace so the current context lists the already-visited children explicitly (the localized DFS encoding of Figure~\ref{fig:trace-encodings}c), restoring DFS to the same level of learnability as BFS.
We call this fix \emph{localization}, and we presuppose a localized trace format throughout the remainder of the paper.

\paragraph{Quantitative signature of scattered retrieval.}
A controlled branching-factor sweep on star trees, where a single root has $k$ leaf children and the verification scope is exactly $k$, gives a quantitative form of the failure (Figure~\ref{fig:pk-sweep}).
All-correct verification accuracy decays as $\hat{r}^k$, where $\hat{r}$ is a per-token retrieval accuracy estimated from the $k{=}4$ point.
The exponential decay matches the multiplicative composition of independent retrieval events the analysis predicts; deviations at $k{=}128$ reflect correlated errors consistent with a latent-quality refinement of the i.i.d.\ model.
Detailed BFS-versus-DFS results, the three encoding variants (simple/mid/verbose), and slot-memory recovery experiments appear in Appendix~\ref{sec:appendix-tree-traversal}.

\begin{figure}[t]
    \centering
    \includegraphics[width=0.65\textwidth]{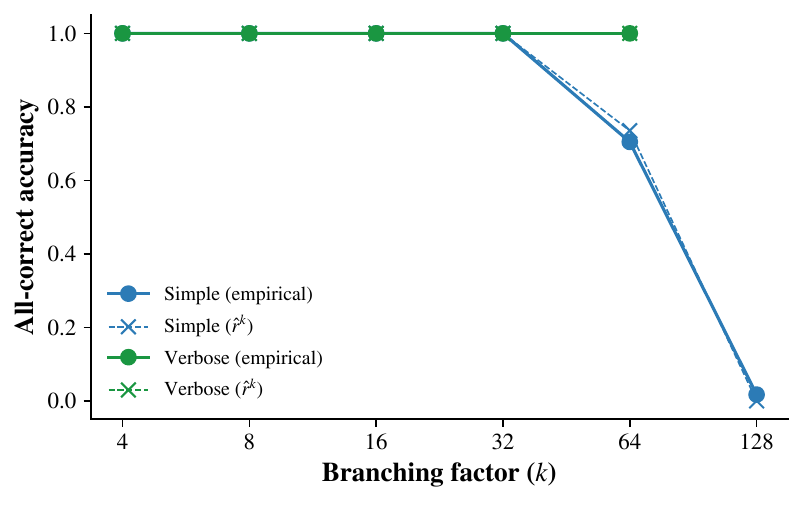}
    \caption{All-correct verification accuracy on star-tree traversal as a function of branching factor $k$. Dashed lines show the $\hat{r}^k$ prediction with per-token retrieval accuracy $\hat{r}$ estimated from the $k{=}4$ point.}
    \label{fig:pk-sweep}
\end{figure}

\paragraph{From tree traversal to backtracking search.}
Tree traversal is the special case of the search-step setting in Definition~\ref{def:problem-setting} with an exhaustive policy and no constraint propagation, and the analysis above isolates one failure mode in this special case: scattered retrieval, an \emph{accessibility} failure that localization repairs by keeping the relevant features local.
Backtracking search adds two ingredients absent from tree traversal: a non-trivial policy that selects which branch to expand next, and constraint propagation that can expose conflicts ahead of exhaustive enumeration.
Both ingredients enrich the trajectory with structure that correlates with the action under the training distribution but is not part of the current state.
The remainder of the paper studies a second, independent failure mode this enrichment exposes: even with localized state, cumulative-trace transformers exploit history features entangled with the action label, learning history-based surrogates rather than state-based verifiers.
This is the \emph{history entanglement} failure introduced in \S\ref{sec:introduction}, and SSA addresses it.
The two fixes act on different levels (Table~\ref{tab:two-failure-modes}): localization changes \emph{where the trace stores state}, isolation changes \emph{what the verifier can read} from it.
A trace-format $\times$ attention-mask factorial confirms they are independently necessary, and the aliasing--entanglement decomposition of Appendix~\ref{sec:appendix-quantitative-theory} captures both within a single formal framework.

\begin{table}[t]
\centering
\caption{The two failure modes this paper addresses, with their respective settings, types, decomposition terms, and fixes.}
\label{tab:two-failure-modes}
\footnotesize
\setlength{\tabcolsep}{4pt}
\begin{tabular}{@{}>{\raggedright\arraybackslash}p{0.18\linewidth} >{\raggedright\arraybackslash}p{0.36\linewidth} >{\raggedright\arraybackslash}p{0.38\linewidth}@{}}
\toprule
& \textbf{Scattered retrieval} & \textbf{History entanglement} \\
\midrule
Setting & Tree traversal (no propagation, exhaustive policy) & Backtracking search (constraint propagation, non-trivial policy) \\
Failure type & Accessibility & Invariance \\
Decomposition term & Aliasing & Entanglement \\
Fix level & Trace design (where the trace stores state) & Attention mask (what the verifier can read) \\
Fix name & Localization & SSA \\
\bottomrule
\end{tabular}
\end{table}

\subsection{Shared Experimental Setup}
\label{sec:slot-memory-method}

\paragraph{Architecture and training.}
We train a decoder-only transformer with $6$ layers, $256$ hidden dimensions, $8$ attention heads, and approximately $6$M parameters on graph coloring (GC; $n{=}30$, $k{=}4$ colors, edge probability $p{=}0.35$) and 3-SAT ($n{=}50$, clause-to-variable ratio $\alpha{=}4.0$). We use \emph{planted} SAT test instances: each test instance is a random formula constructed around a sampled satisfying assignment, so every instance has at least one solution. Each sequence begins with $32$ learnable slot registers, persistent tokens that later positions can read as auxiliary context; the slot ablation (Appendix~\ref{sec:appendix-slot-mask}) shows the main effect we attribute to SSA does not depend on them. Training uses $5{,}000$ solver-generated traces per condition and AdamW~\citep{loshchilov2019decoupled} with $\mathrm{lr}=3\times 10^{-4}$ unless otherwise noted. The same backbone serves all domains; only the tokenizer and vocabulary vary. Blocks World serves as a supporting domain.

\paragraph{Inference protocols and notation.}
The search infrastructure writes a \emph{canonical state block} that encodes $S_t$ as tokens. \emph{Cumulative inference} runs the model autoregressively and appends each new decision block to the growing sequence. \emph{State-rebuilt inference} discards prior decision blocks and feeds only the canonical state block at each step, immediately after the problem prefix. Default autonomous-evaluation token budgets are $4096$ on SAT and Blocks World and $2048$ on graph coloring; the relevant tables flag any deviations. We write $\alpha_v$ for the conditional false-prune rate on viable queried states and reserve unsubscripted $\alpha$ for the SAT clause-to-variable ratio. A consolidated symbol glossary appears in Appendix~\ref{sec:appendix-notation}.

\paragraph{Division of labor.}
The model produces $a_t$ (the policy and verification components jointly) while the infrastructure produces the propagated domains and the tried-alternative set inside $S_t$ (Table~\ref{tab:responsibility}). Although the transformer emits both components of $a_t$ under a single next-token loss, we analyze them as separate prediction problems: reactive verification is a state-only viability decision with a path-invariant Bayes target (Definition~\ref{def:problem-setting}), whereas branching policy is heuristic-dependent and admits multiple satisfying orderings. The two also carry asymmetric failure costs: a false backtrack on a viable state irreversibly prunes the remaining solution path, while a suboptimal branching choice typically leaves alternative branches available (Proposition~\ref{prop:fp-compounding}, Appendix~\ref{sec:precision-barrier}). Symbolic algorithms perform constraint propagation exactly, so isolating the verification signal lets us measure whether a transformer can internalize the state-based continue-or-backtrack decision under the same conditions classical solvers operate in; removing the infrastructure drops every model to negligible solve rate (Table~\ref{tab:capability-levels}, raw-autonomous row). The setting therefore targets the \emph{reactive} level of the capability ladder; ``verification'' throughout this paper refers to the reactive level unless otherwise stated.

\begin{table}[t]
\centering
\caption{Division of labor between model and search infrastructure.}
\label{tab:responsibility}
\footnotesize
\setlength{\tabcolsep}{4pt}
\begin{tabular}{@{}>{\raggedright\arraybackslash}p{0.36\linewidth} >{\raggedright\arraybackslash}p{0.12\linewidth} >{\raggedright\arraybackslash}p{0.46\linewidth}@{}}
\toprule
\textbf{Component} & \textbf{Model} & \textbf{Infrastructure} \\
\midrule
Policy: variable selection & yes & --- \\
Policy: value selection & yes & --- \\
State-block token serialization & --- & yes \\
Unit propagation & --- & yes \\
Conflict exposure after propagation & --- & yes \\
Reactive verification (post-conflict) & yes & --- \\
Backjump mechanics & --- & yes \\
Proactive verification (pre-conflict) & yes & with oracle or CDCL augmentation \\
\bottomrule
\end{tabular}
\end{table}

\begin{table}[t]
\centering
\caption{Capability ladder for dead-end detection; the rightmost column flags each level's role in this paper.}
\label{tab:capability-levels}
\footnotesize
\setlength{\tabcolsep}{4pt}
\begin{tabular}{@{}>{\raggedright\arraybackslash}p{0.14\linewidth} >{\raggedright\arraybackslash}p{0.29\linewidth} >{\raggedright\arraybackslash}p{0.29\linewidth} >{\raggedright\arraybackslash}p{0.21\linewidth}@{}}
\toprule
\textbf{Level} & \textbf{Infrastructure provides} & \textbf{Model learns} & \textbf{Status in this paper} \\
\midrule
Raw autonomous & Raw problem (conjunctive normal form (CNF) or graph) & Propagation $+$ policy $+$ verification & Negligible solve, lower-bound condition \\
\midrule
Reactive & Propagated state $+$ exposed conflict & Policy (variable, value) $+$ reactive verification (continue or backtrack) & \textbf{Main positive result} \\
\midrule
Proactive & Locally consistent state, no conflict & Policy $+$ proactive verification (existential dead-end prediction) & Open problem; failure modes analyzed \\
\bottomrule
\end{tabular}
\end{table}

%% file: ext_theory.tex
\section{Selective State Attention}
\label{sec:ssa}
\label{sec:verification-bottleneck}

With localization established (\S\ref{sec:tree-traversal-recap}), this section addresses the second failure mode---history entanglement---with a fixed attention mask.

At every non-leaf state $S_t$ the model must emit a backtrack token if $S_t$ is a dead end.
The truth value depends on $S_t$ alone (Definition~\ref{def:problem-setting}), but a cumulative-trace-trained causal transformer conditions its prediction on the entire trajectory $(B_1, \ldots, B_t)$ because the next-token objective admits any predictor that fits the joint distribution of states and histories.
We address this with a fixed attention mask that delivers structural invariance to prior trajectory content; Appendix~\ref{sec:appendix-quantitative-theory} formalizes the underlying state-equivalence framework as an aliasing--entanglement decomposition with an information-theoretic refinement, and Appendix~\ref{sec:precision-barrier} formalizes the false-prune threshold $(1-\alpha_v)^M$ that governs precision-sensitive verification.

\begin{proposition}[SSA Conditional Invariance]
\label{prop:ssa-invariance}
Let $P$ denote the problem prefix and $B_t$ the current decision block.
Under selective SSA with block-relative positional encoding (positional indices reset at each block boundary, so the position of a token within $B_t$ does not depend on $t$; Appendix~\ref{sec:appendix-ssa-variants}), the logits at positions inside $B_t$ depend only on $(P, B_t)$.
For any two histories $H_{<t}, H'_{<t}$ they are exactly equal.
\end{proposition}

The mask zeroes attention from $B_t$ to any prior $B_j$, the model writes to slots only while processing $P$, and block-relative positions remove the only remaining channel through which $H_{<t}$ could affect the residual stream at $B_t$.
A padding-control counterfactual (Appendix~\ref{sec:appendix-padding-control}) verifies the proposition empirically: inserting unrelated decision blocks between the prefix and the current state changes SSA's argmax decisions zero times.
SSA therefore eliminates predictor-side history dependence by construction; any cumulative-vs-state-rebuilt gap reflects positional extrapolation rather than history leakage (Appendix~\ref{sec:appendix-padding-control}).
SSA presupposes the localized trace format that addresses the orthogonal aliasing failure (\S\ref{sec:tree-traversal-recap}, Table~\ref{tab:two-failure-modes}); the two interventions are independently necessary for state-rebuilt transfer.

\subsection{Architecture}
\label{sec:ssa-architecture}

SSA satisfies the state-isolation requirement through the attention mask alone, partitioning the input into three regions: a \textbf{problem prefix} (block~0) holding the graph adjacency, constraint definitions, and initial state; \textbf{decision blocks} $B_1, \ldots, B_t$ each containing one step's partial assignment, propagated domains, tried-alternative set, and chosen action ($B_t$ is the current block, $B_1,\ldots,B_{t-1}$ the history; Appendix~\ref{sec:appendix-tokenization} shows the concrete tokenization for both domains); and \textbf{slot registers} $\mathcal{R}$, learnable tokens that provide problem-conditioned auxiliary context (part of the shared backbone, \S\ref{sec:preliminary}; ablated in Appendix~\ref{sec:appendix-slot-mask}).

\paragraph{Attention rules.}
Let $P$ denote the problem-prefix positions, $\mathcal{R}$ the slot-register positions, and $B_i$ the positions in decision block $i$.
The SSA mask (Figure~\ref{fig:ssa-mask}) imposes:
\begin{itemize}[nosep]
    \item Tokens in $P$ attend causally within $P$ and to $\mathcal{R}$.
    \item Tokens in $\mathcal{R}$ attend to all of $P$ and to other tokens in $\mathcal{R}$.
    \item Tokens in $B_i$ attend to all of $P$, all of $\mathcal{R}$, and causally within $B_i$; they do \emph{not} attend to $B_{j}$ for any $j \neq i$.
\end{itemize}

Prefix and slot access are useful for the cumulative protocol rather than essential to the mask topology; crossed ablations in the appendix vary mask topology and positional scheme independently (Appendix~\ref{sec:appendix-ssa-variants}).

\begin{figure}[t]
\centering
\begin{tikzpicture}[scale=0.52, every node/.style={font=\footnotesize}]
% Grid dimensions
\def\n{13}
% Background
\fill[gray!10] (0,0) rectangle (\n,\n);

% Region labels (top)
\draw[decorate, decoration={brace, amplitude=4pt, mirror}] (0,\n+0.2) -- (2,\n+0.2) node[midway, above=4pt] {Slots};
\draw[decorate, decoration={brace, amplitude=4pt, mirror}] (2,\n+0.2) -- (5,\n+0.2) node[midway, above=4pt] {Problem};
\draw[decorate, decoration={brace, amplitude=4pt, mirror}] (5,\n+0.2) -- (8,\n+0.2) node[midway, above=4pt] {Block 1};
\draw[decorate, decoration={brace, amplitude=4pt, mirror}] (8,\n+0.2) -- (11,\n+0.2) node[midway, above=4pt] {Block 2};
\draw[decorate, decoration={brace, amplitude=4pt, mirror}] (11,\n+0.2) -- (13,\n+0.2) node[midway, above=4pt] {Blk 3};

% Region labels (left)
\draw[decorate, decoration={brace, amplitude=4pt}] (-0.2,\n-2) -- (-0.2,\n) node[midway, left=4pt, rotate=90, anchor=south] {Slots};
\draw[decorate, decoration={brace, amplitude=4pt}] (-0.2,\n-5) -- (-0.2,\n-2) node[midway, left=4pt, rotate=90, anchor=south] {Problem};
\draw[decorate, decoration={brace, amplitude=4pt}] (-0.2,\n-8) -- (-0.2,\n-5) node[midway, left=4pt, rotate=90, anchor=south] {Block 1};
\draw[decorate, decoration={brace, amplitude=4pt}] (-0.2,\n-11) -- (-0.2,\n-8) node[midway, left=4pt, rotate=90, anchor=south] {Block 2};
\draw[decorate, decoration={brace, amplitude=4pt}] (-0.2,\n-13) -- (-0.2,\n-11) node[midway, left=4pt, rotate=90, anchor=south] {Blk 3};

% Slots attend to Slots + Problem
\fill[blue!40] (0,\n-2) rectangle (2,\n);    % slots->slots
\fill[blue!40] (2,\n-2) rectangle (5,\n);    % slots->problem

% Problem attends to Slots + Problem (causal)
\fill[blue!40] (0,\n-5) rectangle (2,\n-2);  % problem->slots
\fill[green!40] (2,\n-5) rectangle (5,\n-2); % problem->problem (causal)

% Block 1 attends to Slots + Problem + Block 1 (causal)
\fill[blue!40] (0,\n-8) rectangle (2,\n-5);  % block1->slots
\fill[blue!40] (2,\n-8) rectangle (5,\n-5);  % block1->problem
\fill[orange!40] (5,\n-8) rectangle (8,\n-5); % block1->block1 (causal)
% Block 1 does NOT attend to Block 2, Block 3

% Block 2 attends to Slots + Problem + Block 2 (causal)
\fill[blue!40] (0,\n-11) rectangle (2,\n-8);  % block2->slots
\fill[blue!40] (2,\n-11) rectangle (5,\n-8);  % block2->problem
% Block 2 does NOT attend to Block 1
\fill[orange!40] (8,\n-11) rectangle (11,\n-8); % block2->block2 (causal)

% Block 3 attends to Slots + Problem + Block 3 (causal)
\fill[blue!40] (0,\n-13) rectangle (2,\n-11);  % block3->slots
\fill[blue!40] (2,\n-13) rectangle (5,\n-11);  % block3->problem
\fill[orange!40] (11,\n-13) rectangle (13,\n-11); % block3->block3 (causal)

% Grid lines
\draw[gray!50, very thin] (0,0) grid (\n,\n);

% Region boundaries
\draw[black, thick] (2,0) -- (2,\n);
\draw[black, thick] (5,0) -- (5,\n);
\draw[black, thick] (8,0) -- (8,\n);
\draw[black, thick] (11,0) -- (11,\n);
\draw[black, thick] (0,\n-2) -- (\n,\n-2);
\draw[black, thick] (0,\n-5) -- (\n,\n-5);
\draw[black, thick] (0,\n-8) -- (\n,\n-8);
\draw[black, thick] (0,\n-11) -- (\n,\n-11);

% Outer border
\draw[black, thick] (0,0) rectangle (\n,\n);

% Legend
\fill[blue!40] (14.5,12) rectangle (15.5,12.5); \node[right] at (15.7,12.25) {Full access};
\fill[green!40] (14.5,11) rectangle (15.5,11.5); \node[right] at (15.7,11.25) {Causal (prefix)};
\fill[orange!40] (14.5,10) rectangle (15.5,10.5); \node[right] at (15.7,10.25) {Causal (block)};
\fill[gray!10] (14.5,9) rectangle (15.5,9.5); \node[right] at (15.7,9.25) {Blocked};

% Axis labels
\node[below] at (\n/2, -0.3) {Key (attends to)};
\node[rotate=90] at (-1.2, \n/2) {Query (attending from)};
\end{tikzpicture}
\caption{SSA attention mask. Blue: full access; green: causal within the prefix; orange: causal within the same block; white: blocked.}
\label{fig:ssa-mask}
\end{figure}
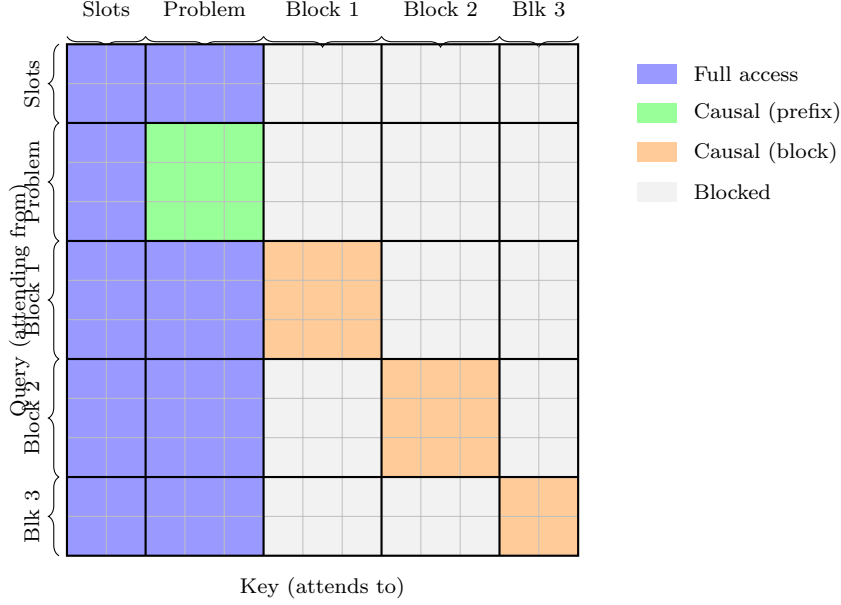

%% file: ext_ssa_experiments.tex
\section{Empirical Evaluation} \label{sec:ssa-experiments}

With localization in the trace (\S\ref{sec:tree-traversal-recap}) and isolation in attention (\S\ref{sec:ssa}) both in place, we now ask whether the resulting model behaves like a state-based verifier.
The four subsections that follow address four questions: does the architectural fix work end-to-end under deployment shift (\emph{state-rebuilt transfer})? Is the invariance structural rather than coincidental (\emph{direct evidence})? Is mask-level isolation specifically necessary or could simpler interventions have produced the result? And how does the result scale across regimes and tasks (\emph{scope and proactive limit})?
End-to-end solve rates confound the policy and verification components, the trace format, and the inference protocol; a verifier-only state-equivalence benchmark on canonical states reached by multiple histories serves as our primary diagnostic, corroborated by the solve-rate and history-transplant results that we report alongside it.
The primary domain is 3-SAT at $n{=}50$ with the default settings of \S\ref{sec:slot-memory-method}; default test conditions are $200$ planted instances at test-instance seed $42$ evaluated across $5$ model seeds, and captions flag any deviations. Supporting domains are 3-SAT $n{=}75$ at the phase transition, graph coloring, Blocks World, and backtracking parsing of an ambiguous expression grammar (Appendix~\ref{sec:appendix-supplementary-main}).
All conditions share the same backbone; only the attention mask varies.

\subsection{State-Rebuilt Transfer}
\label{sec:state-isolation}
\label{sec:ssa-sat}

We train SSA and a causal baseline on the same cumulative traces and evaluate each checkpoint under both inference protocols of \S\ref{sec:slot-memory-method} (Figure~\ref{fig:transplant-rebuild}b), with the input/output split of Table~\ref{tab:responsibility}.

\begin{figure}[t]
\centering
\includegraphics[width=\linewidth]{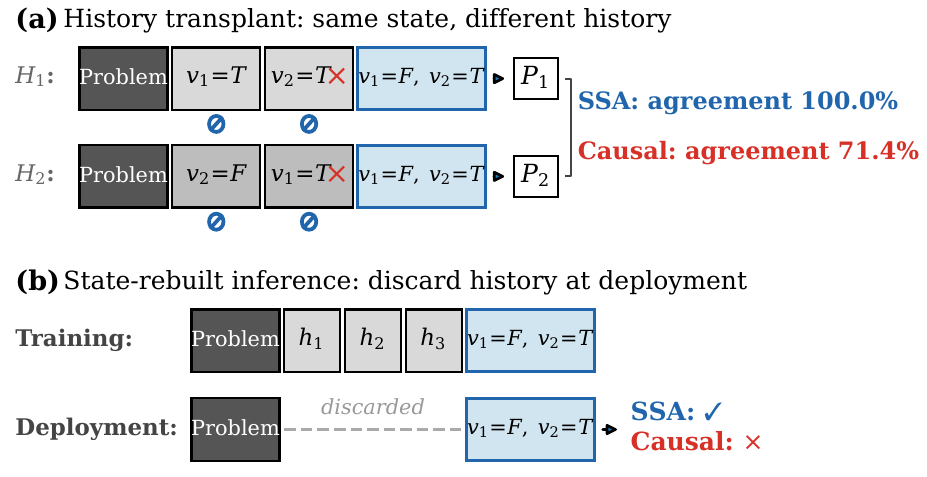}
\caption{Schematic of two constructs introduced in this section. \textbf{(a)} History transplant: two histories $H_1, H_2$ (gray blocks) reach the same canonical state, producing next-step distributions $P_1, P_2$. \textbf{(b)} State-rebuilt inference: deployment context contains only the problem prefix and the current state block.}
\label{fig:transplant-rebuild}
\end{figure}

\begin{table}[!tp]
\centering
\caption{Solve rate (\%) on the primary 3-SAT setting. Table~\ref{tab:training-budget-main} details the \textbf{Train} column; the Panel B baseline conditions are described in later subsections; Appendix~\ref{sec:appendix-multi-instance-seed} provides pooled three-seed estimates, and Appendix~\ref{sec:appendix-verifier-swap} reports verifier-swap and oracle-decomposition analyses. $^{\dagger}$Causal (state-only) 100k uses 3 seeds.}
\label{tab:flagship}
\label{tab:protocol-bridge}
\footnotesize
\setlength{\tabcolsep}{4pt}
\resizebox{\linewidth}{!}{\input{tab_flagship}}
\end{table}

\begin{table}[!tp]
\centering
\caption{Training-budget axes for the Panel A rows: solver-trace count, block-level supervised actions seen by the next-token objective, and native inference protocol.}
\label{tab:training-budget-main}
\footnotesize
\setlength{\tabcolsep}{4pt}
\resizebox{\linewidth}{!}{\input{tab_training_budget}}
\end{table}

Table~\ref{tab:training-budget-main} separates the trace count from the supervision granularity for the Panel~A rows: cumulative-trained transformers receive supervision at every decision block, whereas the state-only rows subsample one $(\text{state}, \text{action})$ triple per trace at the indicated count.
SSA at $5k$ cumulative therefore matches Causal (state-only) $100k$ and the MLP/Factor graph neural network (GNN;~\citealp{scarselli2009graph}) baselines in block-level supervision count, and we report the comparison at matched supervised-action level.

Under cumulative inference SSA improves modestly over causal (Table~\ref{tab:protocol-bridge}A).
Under state-rebuilt inference, SSA holds near its peak solve rate while the cumulative-trained causal transformer drops sharply (Table~\ref{tab:protocol-bridge}A, with pooled multi-instance-seed estimates in Appendix~\ref{sec:appendix-multi-instance-seed} confirming the gap is insensitive to test-instance sampling).

The two protocol comparisons within Table~\ref{tab:protocol-bridge}A reflect two distinct mechanisms rather than a single uniform protocol effect.
The cumulative-trained causal transformer's drop under state-rebuilt inference indicates a predictor that depends on prior-block content, the failure mode SSA targets.
SSA also scores lower under cumulative inference than under state-rebuilt; however, Proposition~\ref{prop:ssa-invariance} forbids any content-history channel inside the current decision block, so the cumulative-vs-state-rebuilt difference for SSA arises from a different source.
Under cumulative inference the context grows with the trajectory and can place tokens inside $B_t$ at positions beyond the range seen at training, while state-rebuilt inference resets the positional range at every step.
The padding-control counterfactual (Appendix~\ref{sec:appendix-padding-control}) attributes the SSA gap to this positional extrapolation rather than to history content, and the history-transplant test below confirms invariance on same-state/different-history pairs at floating-point precision while the cumulative-trained causal transformer disagrees at substantial rates.

When we instead supervise a causal transformer on single-step state--action examples directly, it reaches the same level as SSA at matched supervised-action level (Table~\ref{tab:protocol-bridge}A); SSA therefore recovers the behavior of a dedicated state-only pipeline using only the cumulative traces a normal solver already produces, by changing only the attention mask.
The state-feature MLP approaches the upper bound under state-rebuilt inference (Table~\ref{tab:protocol-bridge}A), so the cumulative-trained causal transformer's collapse reflects a failure to use information already present in the input, not a representational ceiling.
As an additional probe (rightmost column of Panel~A), we replace learned variable selection with uniform random sampling at inference: SSA, the $100k$ state-only transformer, and the MLP all retain their native-protocol solve rates with only minor drops, while the cumulative-trained causal transformer collapses to the random-branching baseline, separating state-based decision functions from history-entangled trajectory imitators.

\subsection{Direct Evidence of History Invariance}
\label{sec:history-transplant}
\label{sec:verifier-only}

The state-rebuilt transfer above is consistent with the central claim that SSA enforces structural invariance to prior trajectory content, but it also admits weaker explanations: a history-content correlation that happens to survive discarding the trajectory, or a positional-encoding artifact that mimics invariance.
Two direct diagnostics below rule both out; a length out-of-distribution (OOD) diagnostic and a counterfactual padding control in Appendix~\ref{sec:appendix-padding-control} reproduce the same conclusion in the cumulative protocol.

\paragraph{History transplant.}
We construct pairs that share the same state but differ in prior history (Figure~\ref{fig:transplant-rebuild}a) by pairing decision points across stochastic rollouts whose canonical state block (\S\ref{sec:slot-memory-method}), conflict status, and decision-stack depth all agree.
We order state lexically by variable identifier so the current decision block becomes an identical token sequence across the two histories, and we run SSA under its default block-relative positions; together these conditions satisfy the hypotheses of Proposition~\ref{prop:ssa-invariance} and isolate the diagnostic to prior-block content.
We report symmetric Kullback--Leibler (KL) divergence between $P_1$ and $P_2$ and argmax-action agreement at the position immediately preceding the action token (Table~\ref{tab:transplant-metrics}).
For SSA, Proposition~\ref{prop:ssa-invariance} guarantees $P_1 = P_2$ exactly, and the empirical numbers match this prediction to floating-point precision on both domains (Table~\ref{tab:transplant-metrics}).
The causal transformer trained on the same traces shows substantial disagreement on both metrics, confirming that cumulative-trace training yields a predictor that depends on prior history content, not on the local state alone.

\begin{table}[t]
\centering
\caption{History-transplant metrics for SSA and cumulative-trained causal attention.}
\label{tab:transplant-metrics}
\footnotesize
\setlength{\tabcolsep}{4pt}
\resizebox{\textwidth}{!}{%
\begin{tabular}{l l c c c}
\toprule
\textbf{Domain} & \textbf{Architecture} & \textbf{Pairs} & \textbf{Argmax agreement (\%)} & \textbf{Mean symmetric KL} \\
\midrule
\multirow{2}{*}{SAT (5 seeds)} & SSA    & $28{,}420$ & $100.000 \pm 0.000$ & $5.6 \times 10^{-14}$ \\
                                       & Causal & $28{,}420$ & $71.44 \pm 7.62$    & $1.26 \pm 0.22$ \\
\midrule
\multirow{2}{*}{GC $n{=}30$}            & SSA    & $458$      & $100.0$             & ${\sim}0$ \\
                                       & Causal & $458$      & $89.1$              & $0.16$ \\
\bottomrule
\end{tabular}}
\end{table}

\paragraph{Verifier-only state-equivalence benchmark.}
The transplant test above compares full action distributions, which entangles the policy and verification components.
We complement it with a benchmark that isolates the verification claim by restricting attention to the binary continue-or-backtrack decision on canonical states reached by multiple histories, using a fixed external probe bank that does not depend on the model under test; Appendix~\ref{sec:appendix-verifier-calibration} gives the probe bank, exact metric definitions, and full per-condition results.
We define the central diagnostic as the protocol gap in area under the receiver operating characteristic curve (AUROC), $\Delta\mathrm{AUROC} = \mathrm{AUROC}^{\,\mathrm{cum}} - \mathrm{AUROC}^{\,\mathrm{SR}}$ between cumulative and state-rebuilt deployments, which equals zero for any predictor that depends only on the canonical state and turns positive whenever the verifier signal weakens under state-rebuilt deployment.
We sweep ten interventions across five sites: mask topology (SSA, current-block-only), recurrent block masking, data-level history reductions (block dropout, sliding window, null history, transplanted donor histories), an objective-level contrastive invariance regularizer, and a long short-term memory (LSTM)~\citep{hochreiter1997long} family member.
Among the swept interventions, only mask-level isolation between decision blocks produces a calibrated history-invariant verifier: SSA and the strictly more restrictive current-block-only mask both attain a zero protocol gap on SAT, while cumulative-trained causal attention loses a substantial fraction of its AUROC under state-rebuilt deployment (Appendix~\ref{sec:appendix-verifier-calibration}).
Data-level history reductions and the contrastive objective do not reproduce the invariance: contrastive regularization keeps the protocol gap small only by collapsing the missed-conflict rate $\beta$ to its degenerate maximum at every threshold (the model emits the continue token at every conflict, so the small protocol gap reflects a constant predictor rather than a calibrated verifier).

\subsection{Why Architectural Isolation Is Necessary}
\label{sec:scaffold-ladder}
\label{sec:alternatives}
\label{sec:ssa-baselines}

The verifier-only benchmark already shows that mask-level cross-block isolation is the only swept intervention that recovers history invariance; this subsection asks the dual question: is isolation alone sufficient, or does it require a particular trace format?
A trace-format $\times$ attention-mask factorial (Appendix~\ref{sec:appendix-scaffold}) crosses three trace formats with SSA and causal masks under both inference protocols.
Only the \emph{enriched $+$ SSA $+$ state-rebuilt} cell solves the search reliably; cells that satisfy localization alone (\emph{enriched $+$ causal}) or isolation alone (\emph{residual-CNF $+$ SSA}, \emph{stripped $+$ SSA}) collapse to near-random performance.
Localization and isolation are therefore independently necessary, and SSA presupposes the localized trace format throughout.

\paragraph{Confounds at the model level.}
Four model-level diagnostics each rule out an alternative explanation for the SSA advantage (Table~\ref{tab:diagnostics}).
A separate mask $\times$ slots factorial (Appendix~\ref{sec:appendix-slot-mask}) shows SSA's advantage over causal is at least as large without slot registers as with them, ruling out the slot registers as the source of the gain.

\begin{table}[t]
\centering
\caption{Model-level confound diagnostics. Each diagnostic targets a specific alternative explanation for SSA's state-rebuilt advantage.}
\label{tab:diagnostics}
\footnotesize
\begin{tabular}{l l l}
\toprule
\textbf{Diagnostic} & \textbf{Result} & \textbf{Rules out} \\
\midrule
Visited-set oracle & Intervenes 0 times & Simple repetition \\
Representation probe & SSA invariant; causal diverges & Superficial attention effect \\
Train from scratch & SSA advantage reproduced & Checkpoint artifact \\
State-augmented causal & $99.1\%$ offline, ${\le}2\%$ inference & Exposure without isolation \\
\bottomrule
\end{tabular}
\end{table}

\paragraph{Alternative routes to invariance fail at the end-to-end level too.}
The intervention sweep of \S\ref{sec:history-transplant} on the verifier-only benchmark already showed that only mask-level cross-block isolation produced a calibrated history-invariant verifier; we extend the comparison to end-to-end solve rate (Table~\ref{tab:protocol-bridge}B) and find the same conclusion.
Data-level history reductions and the objective-level contrastive regularizer all collapse to near-random performance under state-rebuilt inference; the LSTM with SSA-like block masking does not recover transfer at our matched configuration either.
A Factor GNN trained on $(\text{state}, \text{action})$ pairs solves every evaluation instance in a small number of decisions on average, confirming that the difficulty is specific to transformers learning a history-invariant function from serialized cumulative traces, not to the SAT instances themselves.

\subsection{Scope and Proactive Limit}
\label{sec:scope-and-limits}
\label{sec:verification-experiment}
\label{sec:parsing-domain}

\begin{figure}[t]
\centering
\includegraphics[width=\textwidth]{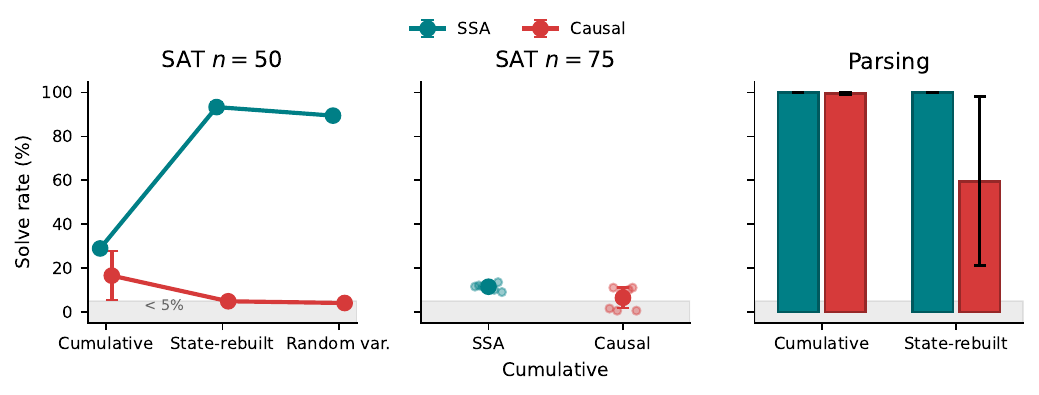}
\caption{Solve rates across three domains (mean $\pm$ std). \textbf{Left:} primary 3-SAT under three inference protocols. \textbf{Center:} 3-SAT $n{=}75$ phase-transition ($8$ seeds, cumulative inference only) with per-seed points. \textbf{Right:} Backtracking parsing under two protocols.}
\label{fig:protocol-response}
\end{figure}

\paragraph{Phase-transition scaling and non-CSP transfer.}
The result holds across two robustness checks (Figure~\ref{fig:protocol-response}).
On 3-SAT at the phase transition~\citep{mitchell1992hard,crawford1996experimental} ($n{=}75$, $\alpha{=}4.26$), the failure mode shifts from a mean-performance gap to a training-stability gap: a non-trivial fraction of causal seeds fail to train productively while SSA seeds remain tightly clustered, consistent with a history-dependent function being harder to learn reliably as trajectories grow.
On backtracking parsing of an ambiguous expression grammar (state block contains only the input tokens, cursor, and parser stack), SSA maintains perfect parse success under both inference protocols while causal drops sharply under state-rebuilt inference, demonstrating that the result extends beyond the CSP pattern.
Graph coloring and Blocks World show the same qualitative pattern (Appendix~\ref{sec:appendix-supplementary-main}, \ref{sec:appendix-bw}).

\paragraph{Verifier residual and the proactive limit.}
SSA's verifier holds its false-prune rate comfortably below the threshold of Proposition~\ref{prop:fp-compounding} on both domains (Table~\ref{tab:verifier-calibration}, Appendix~\ref{sec:appendix-verifier-calibration}; Appendix~\ref{sec:precision-barrier}).
The verifier's missed-conflict rate $\beta$ remains non-negligible on SAT, confining the present work to \emph{reactive} verification on conflicts that propagation has already exposed.
The selection criterion we adopt is robustness under deployment shift rather than peak AUROC under one chosen protocol: a verifier whose accuracy depends on having the training-time trajectory present is by definition a history-dependent surrogate (\S\ref{sec:history-transplant}), and the state-rebuilt protocol exposes the part of the predictor that is state-based.
We treat the proactive target as outside scope; our direct attempts to add a learned proactive verifier produce no benefit at our scale (Appendix~\ref{sec:appendix-proactive-verifier}), and the discussion section considers on-policy data collection as a plausible future direction outside the offline-cumulative-trace setting.
A verifier-swap analysis (Appendix~\ref{sec:appendix-verifier-swap}) further decomposes the gap into policy and verification components: holding the verifier as a perfect oracle, SSA $+$ oracle solves every instance while cumulative-trained causal $+$ oracle leaves a substantial deficit, indicating that history entanglement degrades both components rather than verification alone.

%% file: tab_flagship.tex
% Generated by scripts/tables/gen_tab_flagship.py
% Bolding policy: none. See the script docstring for the rationale.
\begin{tabular}{l c ccc}
\toprule
\multicolumn{5}{l}{\textbf{A.} \textit{Evaluation protocol comparison}} \\
\midrule
& \textbf{Train} & \textbf{Cumulative} & \textbf{State-rebuilt} & \textbf{SR + rand.\ var.} \\
\cmidrule(lr){3-5}
Causal & 5k cum. & $16.5 \pm 11.2$ & $4.8 \pm 1.2$ & $4.0 \pm 0.0$ \\
SSA & 5k cum. & $28.9 \pm 1.9$ & $93.4 \pm 2.3$ & $89.5 \pm 1.1$ \\
Causal (state-only) & 5k state & --- & $69.3 \pm 18.1$ & $44.9 \pm 8.9$ \\
Causal (state-only)$^{\dagger}$ & 100k state & --- & $93.0 \pm 2.0$ & $86.0 \pm 6.9^{\dagger}$ \\
MLP (state features) & 100k state & --- & $99.2 \pm 0.8$ & $100.0 \pm 0.0$ \\
\midrule
\multicolumn{5}{l}{\textbf{B.} \textit{History-reduction and architecture baselines} (all use 5k cumulative traces)} \\
\midrule
& & \textbf{Cumulative} & \textbf{State-rebuilt} & \textbf{$\Delta$ (pp)} \\
\cmidrule(lr){3-5}
SSA & & $28.9 \pm 1.9$ & $93.4 \pm 2.3$ & $+64.5$ \\
Causal & & $16.5 \pm 11.2$ & $4.8 \pm 1.2$ & $-11.7$ \\
Block Dropout $p{=}0.5$ & & $24.8 \pm 4.4$ & $3.8 \pm 1.9$ & $-21.0$ \\
Sliding Window $k{=}3$ & & $5.3 \pm 1.5$ & $3.2 \pm 0.4$ & $-2.1$ \\
Null History & & $4.2 \pm 1.3$ & $4.2 \pm 0.4$ & $+0.1$ \\
LSTM & & $6.8 \pm 3.6$ & $3.1 \pm 0.9$ & $-3.7$ \\
\bottomrule
\end{tabular}

%% file: tab_training_budget.tex
\begin{tabular}{l c c c}
\toprule
\textbf{Method} & \textbf{Solver} & \textbf{Block-level} & \textbf{Native} \\
 & \textbf{traces} & \textbf{supervised actions} & \textbf{inference} \\
\midrule
SSA ($5k$ cum.) & $5{,}000$ & ${\sim}150{,}000$ (full traces) & cum.\ or state-rebuilt \\
Causal ($5k$ cum.) & $5{,}000$ & ${\sim}150{,}000$ (full traces) & cum.\ or state-rebuilt \\
Causal state-only ($5k$ state) & $5{,}000$ & $5{,}000$ (subsampled, $1$/trace) & state-rebuilt \\
Causal state-only ($100k$ state) & $5{,}000$ & $100{,}000$ (subsampled) & state-rebuilt \\
MLP (state features) & $5{,}000$ & ${\sim}150{,}000$ & state-rebuilt \\
Factor GNN & $5{,}000$ & ${\sim}150{,}000$ & state-rebuilt \\
\bottomrule
\end{tabular}

%% file: ext_related_work.tex
\section{Related Work} \label{sec:related-work}

\paragraph{Logic and control in algorithmic reasoning.}
Building on \citet{kowalski1979algorithm}'s logic/control decomposition introduced in \S\ref{sec:introduction}, our finding is that cumulative-trace supervision of the control component admits predictors violating the state-equivalence the logical specification imposes, and SSA is the architectural fix that restores the alignment.

\paragraph{Neural search policies.}
Tree of Thoughts~\citep{yao2023tree} and Language Agent Tree Search~\citep{zhou2024language} use language models to expand, score, or simulate search branches.
Self-Backtracking~\citep{yang2025selfbacktracking}, ASTRO~\citep{kim2025astro}, and \citet{giannoulis2025sudoku} learn backtracking off-policy.
CLRS~\citep{velickovic2022clrs} and scratchpad supervision~\citep{nye2022show} study neural representations of intermediate algorithmic state.
\citet{qin2025backtrack} show that prescribed search structure does not always help and can hurt.
These works ask whether transformers can learn to search from scratch under various supervision regimes; they vary the search procedure, training traces, or supervision signal.
We ask a complementary question: given a propagated state and a model trained on cumulative traces, does the resulting predictor condition its continue-or-backtrack decision on the current state or on the trajectory that produced it?
SSA modifies the attention architecture to impose the state-equivalence invariance that this question requires, rather than to teach a new search procedure; our contribution is identifying the invariance and showing that imposing it changes behaviour, not the particular mask topology we use to impose it.

\paragraph{Algorithmic reasoning and length generalization.}
The algorithmic-reasoning results in transformers cited in \S\ref{sec:introduction} (in-context function classes, worked-example procedures, symbolic linear algebra) concern single-pass function evaluation, while our failure arises in multi-step search where predictions depend on the path used to reach the current state.
Length generalization~\citep{anil2022exploring,jelassi2023length} and looped transformers~\citep{lee2024teaching,giannou2023looped,yang2024looped} address insufficient computation depth, while SSA targets path sensitivity.
The scattered-retrieval failure mode we identify in tree traversal (\S\ref{sec:tree-traversal-recap}) is a multiplicative-error retrieval bottleneck distinct from the additive-error length-extrapolation literature.

\paragraph{Intermediate-state evaluation.}
Three lines of work evaluate nonterminal search states.
Process reward models score local step quality~\citep{lightman2023lets}, value models estimate future solve probability from a partial state~\citep{snell2024scaling}, and dead-end detectors in classical planning predict whether a state is unsolvable~\citep{steinmetz2016towards}.
Outcome-only methods~\citep{cobbe2021training,wang2023selfconsistency} instead evaluate completed candidates, and recent analyses question whether intermediate-step supervision reliably improves internal-node evaluation~\citep{zhang2025lessons}.
We target existential solvability (whether the current state is viable in the sense of \S\ref{sec:preliminary}), distinct from step correctness and expected future success.
SSA is complementary to these supervision choices, because we keep the training objective fixed at next-token prediction and vary only the attention, so we can attribute any improvements to state isolation rather than to richer labels.
Verifier precision is a central constraint because false-prune events compound multiplicatively along search horizons (Proposition~\ref{prop:fp-compounding}), an effect also noted in reasoning-system evaluations~\citep{gao2023scaling,huang2024large}.
DAgger~\citep{ross2011reduction} addresses deployment shift by broadening the training distribution, while SSA targets path-dependent representations inside the model at evaluation of the current state.

\paragraph{Neural-symbolic hybrid systems.}
Our bounded-CDCL experiments pair the SSA model with a symbolic verifier for exact dead-end checks, following the propose-and-check pattern of AlphaProof~\citep{hubert2025alphaproof}, AlphaGeometry~\citep{trinh2024solving}, FunSearch~\citep{romeraparedes2023mathematical}, and Toolformer~\citep{schick2023toolformer}.
In SAT, NeuroBack~\citep{wang2024neuroback} uses a graph neural network to predict phase preferences that bias the variable-value heuristic of an external CDCL solver, while the solver itself remains classical and retains completeness through clause learning and backjumping.
SSA targets a different component of the search loop: whether a transformer can internalize the reactive backtrack decision \emph{itself} from offline traces, rather than guiding a complete symbolic solver from outside.
NeuroBack's neural component improves a CDCL solver's heuristic decisions but leaves the backtracking decision to the solver; SSA learns the backtracking decision under a state-equivalence guarantee (Proposition~\ref{prop:ssa-invariance}), addressing the harder problem of recovering a path-invariant verifier from history-correlated cumulative traces.
The two approaches compose: SSA pairs cleanly with a bounded CDCL sidecar when completeness is required (Appendix~\ref{sec:appendix-verifier-swap}).

\paragraph{Neural satisfiability solving.}
Most neural SAT methods target satisfiability classification~\citep{selsam2018neurosat,cameron2020predicting} or heuristic guidance inside an existing solver~\citep{shi2023satformer,shi2021transformerbased}.
\citet{pan2025backtracking} study whether transformers can learn Davis-Putnam-Logemann-Loveland (DPLL,~\citealp{davis1960computing,davis1962machine}) style reasoning.
In our setting, the transformer generates the branching variable, value assignment, and backtrack tokens autoregressively, while unit propagation, conflict exposure, state-block serialization, and backjump execution remain external (Table~\ref{tab:responsibility}), following the standard CDCL search loop~\citep{een2003extensible,moskewicz2001chaff}.
The question we ask is whether the transformer can internalize the reactive backtrack decision given propagated state, distinct from the orthogonal question of replacing propagation infrastructure with a neural component.

\paragraph{Structured attention and state abstraction.}
SSA's mask is a structured attention pattern, but it differs in purpose from prior work.
System~2 Attention~\citep{weston2023system}, ETC~\citep{ainslie2020etc}, Block Transformer~\citep{ho2024block}, register tokens~\citep{darcet2024vision}, Slot Attention~\citep{locatello2020slotattention}, and Selective Attention~\citep{leviathan2025selective} modify which tokens attend to which, for goals ranging from efficiency to context filtering.
SSA's distinguishing feature is therefore not the use of an attention mask but the invariance the mask imposes: state-equivalence across same-state/different-history pairs, a property that the reactive verification target requires and that ordinary cumulative-trace training admits predictors that violate.
The approach connects to state abstraction in reinforcement learning~\citep{li2006towards}, but SSA enforces the invariance structurally through attention rather than through a learned abstract state.

%% file: ext_analysis.tex
\section{Discussion} \label{sec:analysis-discussion}

\paragraph{What SSA does and does not solve.}
SSA enforces structural invariance to prior trajectory \emph{content}, eliminating the entanglement term in the aliasing--entanglement decomposition (Appendix~\ref{sec:appendix-quantitative-theory}).
It does not address two adjacent issues.
First, \emph{aliasing}: the trace must serialize the relevant state in the form the predictor needs (the localization prerequisite of \S\ref{sec:scaffold-ladder}); SSA presupposes localization and is not an alternative to it.
Second, invariance to prior-block \emph{length} additionally requires block-relative positions; under standard absolute positions SSA retains content invariance but exhibits a small residual KL divergence at out-of-training-support positions (Appendix~\ref{sec:appendix-padding-control}), attributable to positional extrapolation rather than history leakage.

\paragraph{Inference-time context clearing.}
\label{sec:outlook}
The same isolation principle that motivates the SSA mask suggests a protocol-level analogue for pretrained large language models (LLMs) that keeps autoregressive generation but clears and rebuilds the context window at each decision step.
This design would let standard pretrained models benefit from state isolation without architectural change, and is the most direct way to extend our symbolic-domain result to systems where modifying the attention mask is impractical.
The ingredients the analysis depends on are explicit block boundaries and a canonical state representation that the model can read at each step.
These conditions are present in search traces and scratchpad-style reasoning, are partially present in tool-augmented chain-of-thought protocols, and are the natural starting point for an LLM extension of the work.
A context-clearing study on an explicit-state reasoning benchmark is the natural next step.

\paragraph{Open technical questions.}
RoPE~\citep{su2024roformer}, ALiBi~\citep{press2022alibi}, and NoPE~\citep{kazemnejad2023nope} target length extrapolation, a different failure mode from the autonomous-termination plateau we observe; whether they reduce SSA's small absolute-position residual is an open empirical question.
Proactive verification (detecting locally consistent dead ends before any conflict surfaces) requires lookahead reasoning that the offline cumulative-trace setting is poorly suited to learn; on-policy data collection in the style of dataset aggregation (DAgger)~\citep{ross2011reduction} is a plausible direction.

%% file: ext_conclusion.tex
\section{Conclusion} \label{sec:conclusion}

Cumulative-trace-trained causal transformers fail at reactive backtrack verification, learning a history-entangled surrogate that flips its decision on a substantial fraction of same-state pairs and loses verifier accuracy once deployment removes the trajectory.
The failure traces to a mismatch between the cumulative-trace training objective---which admits any predictor that fits the joint distribution of states and histories---and the state-equivalence the verification target requires.

Selective State Attention removes the trajectory route by construction: under its mask, each decision step's prediction depends only on the problem prefix and the current decision block (Proposition~\ref{prop:ssa-invariance}), and the mask adds no parameters and leaves the training objective unchanged.
On the verifier-only state-equivalence benchmark SSA produces identical decisions under cumulative and state-rebuilt deployments; among the ten alternative interventions we test, only mask-level cross-block isolation recovers transfer between the two protocols (\S\ref{sec:alternatives}).
Together with the localization fix from \S\ref{sec:tree-traversal-recap}, this yields a two-piece prescription for cumulative-trace transformers: localize the state in the trace and isolate it in attention; the factorial of \S\ref{sec:scaffold-ladder} shows that neither piece alone restores state-rebuilt transfer.
A small SSA model trained on cumulative traces matches a causal model trained directly on individual state-action examples under state-rebuilt inference (Table~\ref{tab:protocol-bridge}A), recovering state-isolated deployment behavior without changing the training data.

Our benchmarks serve as diagnostics rather than solver-competitive comparisons: rule-based heuristics with the same search infrastructure already solve every instance at the scales we use (Appendix~\ref{sec:appendix-heuristic-detail}), so we make an architectural claim about transformers on serialized trajectory data rather than a solver-advance claim in any particular domain.
Cumulative-trace training is the default regime for any autoregressive model whose context accumulates a record of its own past actions, and wherever the optimal target depends on a current state rather than the path that produced it the same history-entanglement diagnosis applies in principle; the multiplicative compounding of false-prune events (Proposition~\ref{prop:fp-compounding}, Appendix~\ref{sec:precision-barrier}) makes verifier precision a central constraint in the proactive setting (\S\ref{sec:analysis-discussion}) as well.
The symbolic-domain evidence here opens up inference-time context clearing (\S\ref{sec:outlook}) as a natural avenue for extending state isolation to pretrained language models on tasks with explicit state boundaries.

%% file: ext_appendix.tex
\appendix

\section{Notation Glossary}
\label{sec:appendix-notation}

The main text introduces each symbol at its first use; this glossary consolidates them for reference. Section numbers in parentheses give the place of first substantive use.

\begin{itemize}[nosep]
  \item $X$ \quad problem instance (Definition~\ref{def:problem-setting});
  \item $S_t$ \quad search state at step $t$, comprising the partial assignment, propagated domains, and tried-alternative set (Definition~\ref{def:problem-setting}); $S$ denotes the same object with the time subscript dropped (\S\ref{sec:quotient-state});
  \item $a_t$ \quad action at step $t$, $a_t \in \{\textsc{branch}(x, v),\, \textsc{backtrack}\}$ (Definition~\ref{def:problem-setting});
  \item $B_t$ \quad token-serialized decision block encoding $(S_t, a_t)$ (Definition~\ref{def:problem-setting}); $B_1, \ldots, B_{t-1}$ denote prior blocks;
  \item $P$ \quad problem prefix that token-serializes $X$ (Definition~\ref{def:problem-setting}, \S\ref{sec:ssa});
  \item $H_{<t}$ \quad prior history $(B_1, \ldots, B_{t-1})$ (Definition~\ref{def:problem-setting}); $H$ drops the time subscript (\S\ref{sec:quotient-state});
  \item $T$ \quad trace representation $T = \phi(S)$ observed by the model, as analyzed in the framework (\S\ref{sec:quotient-state});
  \item $Y_t$ \quad verification indicator $Y_t := \mathbf{1}[a_t = \textsc{backtrack}]$, the binary backtrack-or-branch target of the verification analysis (Definition~\ref{def:problem-setting}, Eq.~\eqref{eq:Y-definition}); $Y$ drops the time subscript (\S\ref{sec:quotient-state});
  \item $Y_t^{\,*}(S_t)$ \quad reactive verification target, the state-local Bayes-optimal value of $Y_t$ (Definition~\ref{def:problem-setting}, Eq.~\eqref{eq:reactive-rule});
  \item $f$ \quad generic predictor over $(T, H)$ (\S\ref{sec:quotient-state});
  \item $\eta_S, \eta_T, \eta_{T,H}$ \quad Bayes-optimal predictors $\mathbb{E}[Y \mid \cdot]$ at three information levels (\S\ref{sec:quotient-state});
  \item $\Delta_{\mathrm{ent}}(f)$ \quad entanglement of predictor $f$, equal to $\mathbb{E}[\mathrm{Var}(f(T,H) \mid T)]$ (\S\ref{sec:quantitative-history}, Appendix~\ref{sec:appendix-proofs});
  \item $\mathcal{R}$ \quad slot-register positions (\S\ref{sec:ssa});
  \item $\sigma, \sigma'$ \quad search-trajectory variables for the state-equivalence relation $\sigma \sim \sigma'$ (\S\ref{sec:quotient-state});
  \item $\mathcal{T}$ \quad set of search trajectories (\S\ref{sec:quotient-state});
  \item $K_{\text{train}}$ \quad supervised-trace truncation horizon, in decision blocks (Appendix~\ref{sec:appendix-padding-control});
  \item $\lambda$ \quad weight on the contrastive auxiliary loss (\S\ref{sec:alternatives});
  \item $\alpha_v$ \quad conditional false-prune rate on viable queried states (\S\ref{sec:verification-bottleneck}, Appendix~\ref{sec:precision-barrier});
  \item $\alpha$ (without subscript) \quad SAT clause-to-variable ratio (e.g., $\alpha{=}4.0$);
  \item $M$ \quad number of branch points on a solution path (Appendix~\ref{sec:precision-barrier});
  \item $M_{\text{eff}}, R_{\text{eff}}$ \quad effective per-path branch count and effective number of redundant solution paths in the multi-path corruption model (Appendix~\ref{sec:precision-barrier});
  \item $\tau$ \quad classifier/verifier threshold (\S\ref{sec:verification-experiment}, Appendix~\ref{sec:precision-barrier});
  \item $\pi$ \quad traversal policy over the search tree (Appendix~\ref{sec:precision-barrier});
  \item $\rho$ \quad viable-class prevalence $P(\text{viable} \mid \text{queried})$ (Appendix~\ref{sec:precision-barrier});
  \item $\beta$ \quad missed-conflict rate (Appendix~\ref{sec:precision-barrier});
  \item $\Gamma$ \quad finite rooted search tree; $\Gamma_u$ the subtree at node $u$ (Appendix~\ref{sec:precision-barrier});
  \item $\mathcal{P}$ \quad a solution path through the search tree (Appendix~\ref{sec:precision-barrier});
  \item $\hat{r}$ \quad estimated per-token retrieval accuracy in the tree-traversal $\hat{r}^k$ model (Appendix~\ref{sec:appendix-tree-traversal}).
\end{itemize}

\noindent Locally introduced symbols (e.g.\ block-relative positional indices and linear-probe weights) are defined at point of use.

\section{Formal Framework: State-Equivalence and the Value of History}
\label{sec:appendix-quantitative-theory}
\label{sec:appendix-proofs}
\label{sec:quotient-state}
\label{sec:quantitative-history}
\label{sec:design-requirements}

This appendix formalizes the state-equivalence framework that motivates SSA in \S\ref{sec:ssa}.
Drop the time subscript and write $S$ for the current search state, $H$ for the prior history, and $Y$ for the verification indicator $Y_t$ (Definition~\ref{def:problem-setting}, Eq.~\eqref{eq:Y-definition}).
Two trajectories are \emph{state-equivalent} iff they induce the same $S$, so the truth at internal nodes is $\eta_S := \mathbb{E}[Y \mid S]$, a function of equivalence classes of trajectories rather than of raw token histories.
The model observes a (possibly lossy) trace $T = \phi(S)$ together with $H$.
Let $\eta_T := \mathbb{E}[Y \mid T]$ and $\eta_{T,H} := \mathbb{E}[Y \mid T, H]$ denote the Bayes-optimal predictors at the trace level and the (trace, history) level respectively.

\paragraph{Aliasing--entanglement decomposition.}
Under \emph{state sufficiency} ($Y \perp H \mid S$, automatic from Definition~\ref{def:problem-setting}) and \emph{trace-level history irrelevance} ($Y \perp H \mid T$, equivalently $\eta_{T,H} = \eta_T$), the expected squared loss of any predictor $f(T,H)$ decomposes without cross terms as
\begin{multline}
\label{eq:aliasing-entanglement}
\mathbb{E}[(Y - f(T,H))^2]
= \underbrace{\mathbb{E}[\mathrm{Var}(Y \mid S)]}_{\text{irreducible}}
+ \underbrace{\mathbb{E}[(\eta_S - \eta_T)^2]}_{\text{aliasing}} \\
+ \underbrace{\mathbb{E}[(\eta_T - \mathbb{E}[f \mid T])^2]}_{\text{approximation}}
+ \underbrace{\mathbb{E}[\mathrm{Var}(f \mid T)]}_{\text{entanglement}}.
\end{multline}
\emph{Aliasing} measures state information missing from the trace encoding, including the computational manifestation of \emph{scattered retrieval} in tree traversal (\S\ref{sec:tree-traversal-recap}), where features are present but distributed across positions that bounded per-token attention cannot reliably recompose.
\emph{Entanglement} measures how much $f$'s prediction varies across trajectories that produce the same $T$.
Under causal attention $f$ attends to previous-block tokens, so entanglement is generally non-zero; under SSA, Proposition~\ref{prop:ssa-invariance} gives $f_{\mathrm{SSA}}(T,H) = g(T)$ for some measurable $g$, so $\mathbb{E}[f_{\mathrm{SSA}} \mid T] = f_{\mathrm{SSA}}$ and the entanglement term vanishes by construction.
When trace-level history irrelevance fails, the decomposition carries an additional cross term $-2\,\mathbb{E}[(\eta_{T,H} - \eta_T)(f - \mathbb{E}[f\mid T])]$, identically zero whenever $f$ is $T$-measurable; SSA delivers exactly this measurability.
The decomposition gives two design requirements: \emph{localization} (the trace must expose enough local state to minimize aliasing) and \emph{state isolation} (the attention pattern must block irrelevant trajectory history to minimize entanglement); SSA achieves the latter by construction.

The remainder of this appendix states and proves three additional quantitative results connecting the framework to the empirical diagnostics in \S\ref{sec:history-transplant}.
All hold for a generic predictor $f$ observing a trace $T$ and history $H$, with no assumptions about the domain, model architecture, or search algorithm.

\begin{theorem}[Transplant Variance Identity]
\label{thm:transplant-variance}
Let $f(T,H)$ be any square-integrable predictor.
Let $H'$ be an independent redraw of the trajectory history, keeping the current trace $T$ fixed (formally: $H' \perp H \mid T$ with $H' \mid T \overset{d}{=} H \mid T$).
Define the predictor-side entanglement as the expected variance of $f$ over histories that share the same trace:
\[
\Delta_{\mathrm{ent}}(f) := \mathbb{E}\big[\mathrm{Var}(f(T,H) \mid T)\big].
\]
Then entanglement equals half the expected squared disagreement between predictions under the original and transplanted history:
\begin{equation}
\label{eq:transplant-identity}
\Delta_{\mathrm{ent}}(f) = \frac{1}{2}\,\mathbb{E}\!\left[(f(T,H) - f(T,H'))^2\right].
\end{equation}
\end{theorem}

\begin{proof}
Condition on $T$ and write $X = f(T,H)$, $X' = f(T,H')$.
By construction, $X$ and $X'$ are conditionally i.i.d.\ given $T$.
Hence $\mathbb{E}[(X - X')^2 \mid T] = 2\,\mathrm{Var}(X \mid T)$.
Taking expectations over $T$ gives the result.
\end{proof}

\noindent This identity connects to the transplant diagnostic in \S\ref{sec:history-transplant}: we take a search state, pair it with a different trajectory history, and check whether the model's predictions change.
The fraction of predictions that agree is a thresholded empirical proxy for $\Delta_{\mathrm{ent}}(f)$.

\begin{theorem}[Information Bound on History Value]
\label{thm:info-bound}
Let $R^*(T)$ and $R^*(T,H)$ be the Bayes-optimal $0$--$1$ risks using $T$ alone and $(T,H)$:
\[
R^*(X) := \inf_g \Pr(g(X) \neq Y).
\]
Then
\begin{equation}
\label{eq:info-bound}
0 \le R^*(T) - R^*(T,H) \le \sqrt{\tfrac{1}{2}\,I(Y; H \mid T)},
\end{equation}
where $I(Y; H \mid T)$ is the conditional mutual information between $Y$ and $H$ given $T$.
\end{theorem}

\begin{proof}
Let $p = \eta_T$ and $q = \eta_{T,H}$.
The Bayes risk at posterior $u$ is $r(u) = \min\{u, 1-u\} = \frac{1}{2} - |u - \frac{1}{2}|$.
Since $|a| - |b| \le |a - b|$,
\[
R^*(T) - R^*(T,H) = \mathbb{E}[|q - \tfrac{1}{2}| - |p - \tfrac{1}{2}|] \le \mathbb{E}|q - p|.
\]
For binary $Y$, $|q - p|$ is the total variation between $\mathrm{Bern}(q)$ and $\mathrm{Bern}(p)$.
By Pinsker's inequality~\citep{tsybakov2009nonparametric} and Jensen's inequality,
\[
\mathbb{E}|q - p| \le \sqrt{\tfrac{1}{2}\,\mathbb{E}\!\left[\mathrm{KL}(\mathrm{Bern}(q) \| \mathrm{Bern}(p))\right]} = \sqrt{\tfrac{1}{2}\,I(Y; H \mid T)}.
\]
\end{proof}

\begin{theorem}[Log-Loss Characterization]
\label{thm:logloss}
Let $L^*_{\log}(X) := \inf_g \mathbb{E}[\ell_{\log}(g(X), Y)]$ be the Bayes-optimal log-loss. Then
\begin{equation}
\label{eq:logloss-identity}
L^*_{\log}(T) - L^*_{\log}(T,H) = I(Y; H \mid T).
\end{equation}
\end{theorem}

\begin{proof}
The optimal predictor for log-loss given features $X$ is $g^*(X) = \Pr(Y{=}1 \mid X)$, achieving $L^*_{\log}(X) = \mathbb{H}(Y \mid X)$ where $\mathbb{H}$ denotes Shannon entropy~\citep{cover2006elements}.
Substituting $X = T$ and $X = (T,H)$ gives $L^*_{\log}(T) - L^*_{\log}(T,H) = \mathbb{H}(Y \mid T) - \mathbb{H}(Y \mid T,H) = I(Y; H \mid T)$.
\end{proof}

\paragraph{Scope of Theorems~\ref{thm:info-bound} and~\ref{thm:logloss}.}
Theorem~\ref{thm:info-bound} is stated for binary labels $Y \in \{0,1\}$ under $0$--$1$ risk, which directly covers the continue-or-backtrack verification target that motivates history entanglement in this paper.
The multiclass action components (variable choice and value choice) are not covered by the exact bound $\sqrt{I/2}$.
An analogous bound holds up to a constant factor that depends on the number of classes via Fano-type inequalities~\citep{cover2006elements}; the empirical analysis in this paper uses only the binary case.
Theorem~\ref{thm:logloss} as written targets the same binary setting.
The identity $L^*_{\log}(T) - L^*_{\log}(T,H) = I(Y;H \mid T)$ extends verbatim to arbitrary discrete $Y$ because the optimal log-loss predictor is always the conditional distribution.

\begin{corollary}[SSA Zeroes Entanglement]
\label{cor:ssa-zero-ent}
Let $f_{\mathrm{SSA}}$ be any predictor computed under SSA attention.
By the non-interference property (Eq.~\eqref{eq:non-interference}), $f_{\mathrm{SSA}}(T,H) = g(T)$ a.s.\ for some measurable $g$.
Therefore $\Delta_{\mathrm{ent}}(f_{\mathrm{SSA}}) = 0$, and the Bayes-optimal $0$--$1$ cost of discarding history is bounded by $\sqrt{I(Y; H \mid T)/2}$.
The SSA mask prevents direct attention to trajectory history, so entanglement is zero within the formal model.
In practice, indirect leakage through positional encodings or trace formatting may introduce residual entanglement.
Under our full-state transplant protocol with lexical STATE ordering (\S\ref{sec:history-transplant}), SSA produces identical predictions on every one of $28{,}420$ pairs across $5$ SAT seeds and on every pair on graph coloring, with KL divergence at floating-point precision.
\end{corollary}

\paragraph{Non-interference under SSA.}
Let $B_k$ denote the $k$-th decision block (see \S\ref{sec:ssa}).
For any weights $\theta$, position $i$ in $B_k$ ($k > 0$), and layer $\ell$, the hidden representation $h_i^{(\ell)}$ depends only on the problem prefix and block-local context:
\begin{equation}
\label{eq:non-interference}
    h_i^{(\ell)} = \Phi_{\theta,\ell}(P, B_{k,\le i}).
\end{equation}
The proof follows by induction on layers.
Slot states depend only on $P$.
Positions in $B_k$ attend only to $P$, slots, and $B_{k,\le i}$.
All operations are pointwise deterministic.
Trajectory history cannot influence the current block's logits, regardless of training quality.

\paragraph{Connection to experiments.}
Theorem~\ref{thm:transplant-variance} formalizes the history-transplant diagnostic (\S\ref{sec:history-transplant}); the body of \S\ref{sec:quantitative-history} states the resulting $\Delta_{\mathrm{ent}} \approx 0$ versus large $\Delta_{\mathrm{ent}}$ contrast.
Theorems~\ref{thm:info-bound} and~\ref{thm:logloss} separate \emph{useful} history from \emph{entangled} history.
Table~\ref{tab:history-irrelevance} shows the empirical lift from adding history to a linear probe is small.
This bounds a lower-order component of $I(Y; H \mid T)$ rather than the full conditional mutual information, since non-linear dependence is not captured by the probe.

\section{Supplementary Experiments for Section~\ref{sec:ssa-experiments}}
\label{sec:appendix-supplementary-main}

\subsection{Training-Budget Details}
\label{sec:appendix-training-budget}

Expanding on Table~\ref{tab:training-budget-main} in the main text: optimizer steps and batch size are matched within each training group, the MLP and Factor GNN baselines consume the full set of state-action pairs across all $5{,}000$ traces, and all rows train to convergence as measured by validation-loss plateau.

\subsection{SAT $n{=}30$ Results (Single Seed)}
\label{sec:appendix-sat-n30}

\begin{table}[t]
\centering
\caption{SSA on 3-SAT $n{=}30$ (single seed, autonomous evaluation).}
\label{tab:ssa-sat}
\footnotesize
\begin{tabular}{l ccccc}
\toprule
& \textbf{Solve} & \textbf{Timeout} & \textbf{False UNSAT} & \textbf{Mean Dec} & \textbf{Mean BT} \\
\midrule
\multicolumn{6}{l}{\textit{Budget 4096}} \\
\quad SSA & $\mathbf{93.5}$ & 6.5 & $\mathbf{0.0}$ & 22.5 & 14.1 \\
\quad Causal & 62.5 & 2.0 & 35.5 & 18.5 & 11.3 \\
\midrule
\multicolumn{6}{l}{\textit{Budget 8192}} \\
\quad SSA & $\mathbf{99.0}$ & 1.0 & $\mathbf{0.0}$ & 25.4 & 16.6 \\
\bottomrule
\end{tabular}
\end{table}

SSA substantially outperforms causal attention at $n{=}30$ and further improves at the larger budget (Table~\ref{tab:ssa-sat}).

\subsection{Trace-Format $\times$ Mask Factorial}
\label{sec:appendix-scaffold}

\begin{table}[t]
\centering
\caption{Trace format $\times$ attention mask, both inference protocols. Stripped $+$ state-rebuilt cells are undefined (single-block input requires per-variable annotations).}
\label{tab:scaffold-ladder}
\footnotesize
\resizebox{\textwidth}{!}{\input{tab_scaffold_ladder}}
\end{table}

The factorial in Table~\ref{tab:scaffold-ladder} supports the claim in \S\ref{sec:scaffold-ladder} that localization and isolation are independently necessary.
The \emph{enriched $+$ causal $+$ state-rebuilt} cell satisfies only localization and drops to near-random performance because the cumulative-trained causal transformer learns a function that depends on prior blocks (history entanglement) and state-rebuilt inference removes those blocks.
The \emph{residual-CNF $+$ SSA} and \emph{stripped $+$ SSA} cells satisfy only isolation and remain near random regardless of protocol; with per-variable state absent from the local block, the SSA mask blocks the cross-block attention path that causal would otherwise use to recover it.
The residual-CNF row additionally shows that localization constrains encoding granularity, not only which variables are present: the block must serialize the state in the form most directly relevant to the prediction target, not merely a form from which it is mathematically recoverable.

\subsection{Rule-Based Heuristic Baselines}
\label{sec:appendix-heuristic-detail}

\begin{table}[t]
\centering
\caption{Rule-based heuristic baselines vs.\ neural models on SAT $n{=}50$ and $n{=}75$. Heuristics query the environment directly; neural models read enriched state blocks under state-rebuilt inference.}
\label{tab:heuristic-baselines}
\setlength{\tabcolsep}{4pt}
\footnotesize
\input{tab_heuristic_baselines}
\end{table}

Three rule-based heuristics (occurrence $+$ domain, variable state independent decaying sum (VSIDS,~\citealp{moskewicz2001chaff}) $+$ domain, pure random) use the same search infrastructure but replace all neural predictions with deterministic rules and perform only reactive backtracking.
Structured heuristics solve every instance at every scale we test.
Pure random degrades at $n{=}75$ where backtracking depth stresses the fixed branching rule.
At budget 100 heuristic quality matters and performance separates.

\subsection{Verifier-Swap Analysis: Verification Methods and Oracle Decomposition}
\label{sec:appendix-verifier-swap}

This appendix decomposes the contribution of verification quality to end-to-end solve rate, complementing the verifier-only benchmark of \S\ref{sec:history-transplant} and the false-prune threshold of Appendix~\ref{sec:precision-barrier}.
Two analyses share a common protocol: each instance proceeds through the standard search infrastructure (unit propagation, backjump execution), and we vary either the verifier source while holding the branching policy fixed (Panel~A) or the branching policy while holding the verifier as a perfect oracle (Panel~B).

\paragraph{Panel A: verifier source.}
Holding the SSA branching policy fixed, we replace the source of the backtrack signal with five alternatives.
\emph{Propagation only} disables every learned verifier and lets the search backtrack only when constraint propagation exposes a contradiction.
\emph{Bounded CDCL sidecar (50)} and \emph{Bounded CDCL sidecar (100)} run a symbolic CDCL solver up to the indicated conflict budget at each branch point and adopt its dead-end verdict.
\emph{SSA learned backtrack token} reads the verification decision from the SSA model directly, the configuration used in \S\ref{sec:state-isolation}.
\emph{Perfect dead-end oracle} consults a complete symbolic dead-end check.
The \emph{gap closed} column reports the fraction of the perfect-oracle solve-rate gap that each verifier recovers, $(\text{solve} - \text{solve}_{\text{none}}) / (\text{solve}_{\text{oracle}} - \text{solve}_{\text{none}})$.
The SSA learned backtrack token closes a larger fraction of the oracle gap than the bounded CDCL sidecar, and combining the two recovers most of the remaining gap.

\paragraph{Panel B: branching policy.}
Holding the verifier as a perfect oracle, we vary the branching policy between SSA and cumulative-trained causal attention.
SSA $+$ oracle solves every instance.
Cumulative-trained causal $+$ oracle leaves a substantial deficit that perfect verification cannot repair, indicating that history entanglement degrades both the policy component and the verification component of the search.
The lift column reports the absolute solve-rate gain from adding the perfect oracle, decomposing the contribution of the verifier from the contribution of the branching policy.

\begin{table}[t]
\centering
\caption{Verifier-swap analysis on the primary 3-SAT setting.}
\label{tab:verifier-swap}
\footnotesize
\setlength{\tabcolsep}{4pt}
\resizebox{\linewidth}{!}{\input{tab_verifier_swap}}
\end{table}

\section{Architecture Details}
\label{sec:appendix-architecture}

\subsection{Slot Memory Design}

The slot memory architecture augments a decoder-only transformer with learnable auxiliary registers, injecting slots in the upper half of the network (layers $L/2$ to $L$).
Each augmented layer executes a Read-Process-Write cycle:
\begin{itemize}
    \item \textbf{Read:} slots cross-attend to token representations.
    \item \textbf{Process:} slot self-attention and a feed-forward network enable inter-slot communication.
    \item \textbf{Write:} tokens retrieve verification state via gated cross-attention (gate initialized near zero).
\end{itemize}
\textbf{The main SSA experiments (\S\ref{sec:ssa-experiments}) use only the next-token prediction objective over trace tokens, and no separate verification logit is optimized}; backtracking decisions during autonomous evaluation are read off the same token distribution as variable and value choices.
The delta-local variant in \S\ref{sec:delta-local} is the only configuration that adds an auxiliary verification head (with a full-sequence bidirectional pass).

\paragraph{CSP slot decoder.}
We prepend the learnable slot tokens to every input sequence.
The next-token objective requires that the prediction at any sequence position $i$ depends only on tokens at positions $\le i$; we enforce this by running each forward pass with the causal prefix truncated to position $i$, so slots never read future target tokens.
Slots under SSA act as problem-conditioned latent variables rather than cross-block summaries (\S\ref{sec:ssa}).

\paragraph{Training.}
Optimizer per \S\ref{sec:slot-memory-method}, plus weight decay 0.01, 30--50 epochs, batch size 4--16, dropout 0.1.
Seeds: 42 for default runs. For multi-seed evaluation, we use \{42, 123, 456, 789, 2024\}.

\subsection{Delta-Local Verification Head}
\label{sec:delta-local}

Conflict detection in graph coloring is a delta-local property.
After assigning node $u$, a new domain wipeout can only occur at an uncolored neighbor.
The delta-local head replaces the global verification head with per-neighbor predictions, aggregated via noisy-OR.
$p_{\text{conflict}} = 1 - \prod_{v \in N(u)} (1 - p_{\text{empty}}(v))$.
This reduces effective scope from $O(n)$ to $O(\text{deg}(u))$, consistent with the scope-reduction intuition from Section~\ref{sec:verification-bottleneck}.
Training on oracle traces alone yields high offline precision but low on-policy precision.
Hard-negative mining on model-visited states addresses this discrepancy.
With an oracle policy, the delta-local verifier captures most of the oracle upper bound.
Training uses focal loss ($\gamma{=}2.0$) on local labels and weighted binary cross-entropy on the global label.

\section{Trace Tokenization Examples}
\label{sec:appendix-tokenization}

This section gives concrete examples of the trace formats that the abstract decision blocks $B_t$ in Definition~\ref{def:problem-setting} and the SSA architecture in \S\ref{sec:ssa-architecture} act on.
The examples are produced by the same tokenizer code paths used at training and inference time, so the structure of each block matches the unit of attention isolation enforced by the SSA mask.

\paragraph{Vocabulary conventions.}
Special section markers are wrapped in brackets ($[\texttt{BOS}]$, $[\texttt{CLAUSES}]$, $[\texttt{GRAPH}]$, $[\texttt{SEARCH}]$, $[\texttt{PROP}]$, $[/\texttt{PROP}]$).
SAT uses signed literal tokens \texttt{+v$i$} and \texttt{-v$i$}, unsigned variable tokens \texttt{v$i$}, value tokens \texttt{T} (true), \texttt{F} (false), and \texttt{U} (unassigned), and verdict tokens \texttt{SAT\_OK}, \texttt{UNIT}, \texttt{CONFLICT}.
Graph coloring uses node tokens \texttt{N$i$}, color tokens \texttt{C$j$}, depth tokens \texttt{D$d$}, and bit-mask tokens \texttt{M$bbbb$} that encode the four-bit available-color set.
\texttt{STATE} opens a decision block, \texttt{SEP} separates intra-block fields, and \texttt{[BOS]}/\texttt{[EOS]} bookend the sequence.

\paragraph{Anatomy of a decision block.}
After the problem prefix $P$, every decision block $B_t$ contains four fields:
\begin{itemize}[nosep]
    \item \emph{candidate-set summary} (open after \texttt{STATE}): the unassigned variables (or nodes) and, for graph coloring, their domain sizes;
    \item \emph{propagation evidence} (between \texttt{[PROP]} and \texttt{[/PROP]}): the literal-by-literal verdict on a single inspected clause (SAT) or the per-color blocker analysis at the inspected node (graph coloring);
    \item \emph{action} (\texttt{v$i$ T} / \texttt{v$i$ F} / \texttt{N$i$ M$bbbb$ C$j$}): the variable--value pair the policy emits;
    \item \emph{outcome} (\texttt{OK}, \texttt{CONFLICT C$j$ BJ L$\ell$}, \texttt{SOLVED}, or \texttt{FAILED}): whether propagation accepts the assignment, exposes a conflict and triggers a backjump to level $\ell$, completes the search, or exhausts the tree.
\end{itemize}
A backtrack action is the \texttt{CONFLICT \dots BJ L$\ell$} suffix on a conflicted block; this is the binary continue-or-backtrack decision $Y_t$ (Equation~\eqref{eq:Y-definition}) the verification analysis in \S\ref{sec:verification-bottleneck} targets.
The SSA mask blocks attention from any token in $B_t$ to any token in a different block $B_j$ ($j \neq t$), and Proposition~\ref{prop:ssa-invariance} guarantees that hidden states inside $B_t$ depend only on $(P, B_t)$.

\input{tab_tokenization_examples}

\paragraph{Mapping to the trace-format factorial.}
The SAT example above is the \emph{enriched} format used by default in \S\ref{sec:state-isolation}: each \texttt{STATE} field lists the unassigned variables together with their post-propagation \texttt{T}/\texttt{F}/\texttt{U} domain annotations.
The \emph{stripped} variant (\S\ref{sec:scaffold-ladder}, Table~\ref{tab:scaffold-ladder}) replaces those domain tokens with a single masked placeholder \texttt{?}, leaving the candidate identifiers but removing per-variable forced/free status.
The \emph{residual-CNF} variant replaces the \texttt{STATE} field with a re-serialization of the simplified formula under the current partial assignment, encoding the same information at clause-level rather than per-variable granularity.
The graph-coloring \emph{enriched} variant (Appendix~\ref{sec:appendix-factorial}) adds locally valid colors to each block; the \texttt{M$bbbb$} mask token in the example above already exposes the per-node available-color set, so the local-legal token adds little when the most constrained node already has a unique remaining color.
The factorial in \S\ref{sec:scaffold-ladder} crosses these format axes against the SSA and causal masks, isolating the contributions of localization (controlled by which fields appear in \texttt{STATE}) and isolation (controlled by the attention mask on the boundaries shown above).

\section{Tree Traversal Experiments}
\label{sec:appendix-tree-traversal}

This appendix reports the full tree-traversal experiments that ground the localization fix of \S\ref{sec:tree-traversal-recap}.
The setting isolates the verification component as a binary continue-or-return decision at each parent node, with scope controlled by the branching factor $k$.

\subsection{BFS-versus-DFS Asymmetry}
\label{sec:appendix-bfs-dfs}

We generate random rooted trees with $20$ nodes drawn uniformly over branching factors $k \in \{2,3,4\}$ and train a $4$-layer decoder-only transformer ($d_\mathrm{model}{=}256$, $4$ heads, $\sim 1.5$M parameters) on $50{,}000$ traces per traversal order using next-token cross-entropy.
The same backbone trains separately on BFS and DFS traces.
A trace is \emph{perfect} if every emitted token matches the ground-truth visit order; we report the perfect-trace rate by epoch and the average correct-prefix length, the standard metric in autoregressive sequence modeling that measures how far the model gets before its first divergence.

\begin{figure}[t]
    \centering
    \begin{subfigure}[b]{0.48\linewidth}
        \centering
        \includegraphics[width=\linewidth]{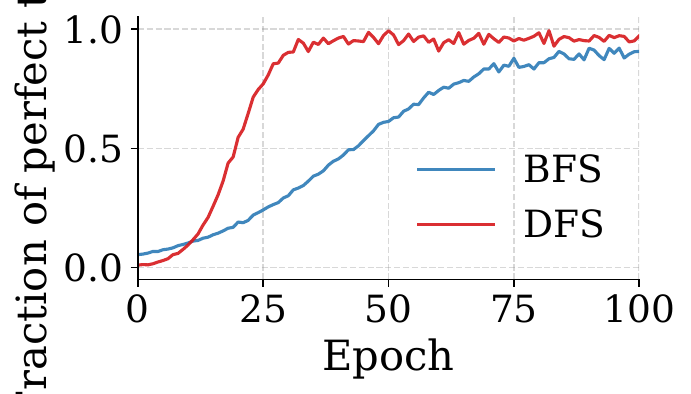}
        \caption{Perfect-trace rate.}
        \label{fig:bfs-dfs-perfect}
    \end{subfigure}
    \hfill
    \begin{subfigure}[b]{0.48\linewidth}
        \centering
        \includegraphics[width=\linewidth]{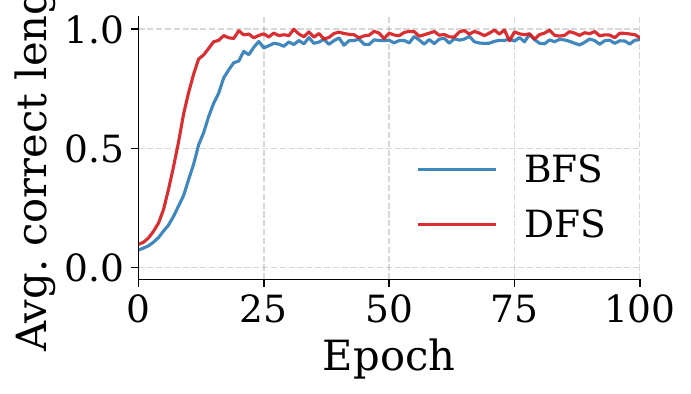}
        \caption{Average correct-prefix length.}
        \label{fig:bfs-dfs-correct}
    \end{subfigure}
    \caption{BFS and DFS training curves on tree traversal under the natural encoding.}
    \label{fig:bfs-dfs}
\end{figure}

DFS spends substantially more epochs in a low-accuracy regime than BFS even though both orders are deterministic functions of the input tree (Figure~\ref{fig:bfs-dfs}).
The asymmetry comes from the verification component, not the policy: BFS verification (is the queue empty?) reads from the most recent token, while DFS verification (have all children of the current parent been visited?) requires recombining $k$ pieces of information distributed across the trajectory.

\subsection{Trace Encodings and the Localization Fix}
\label{sec:appendix-trace-encodings}

Three encoding variants vary the visibility of the visited-children record in the local context.
The \emph{simple} encoding emits node identifiers in visit order with no auxiliary annotation; the visited-children record is implicit in the trajectory.
The \emph{mid} encoding emits a partial annotation that lists the current node's parent at each step but not the siblings.
The \emph{verbose} encoding (the localized variant of Figure~\ref{fig:trace-encodings}c) lists at each return the siblings already visited, exposing the verification scope as a contiguous span in the local context.
The verbose encoding is mathematically redundant given the simple encoding, but the redundancy localizes the relevant features in the current decision block.

\subsection{Star-Tree Sweep and the $\hat{r}^k$ Decay}

To control the verification scope precisely we use star trees where a single root has $k$ leaf children, sweep $k \in \{4, 8, 16, 32, 64, 128\}$, and train one model per $(k, \text{encoding})$ combination on $50{,}000$ trees.
At each return to the root the model decides whether all $k$ children have been visited; this is exactly the binary verification target with controlled scope $k$.
Under the simple encoding all-correct verification accuracy decays as $\hat{r}^k$ (Figure~\ref{fig:pk-sweep}, \S\ref{sec:tree-traversal-recap}) where $\hat{r}$ is the per-token retrieval accuracy estimated from the $k{=}4$ point: $\hat{r}$ remains close to one across most of the swept range, with a deviation only at the largest $k$ consistent with correlated retrieval errors at high scope.
At a fixed $k$, increasing trace verbosity from simple to mid to verbose monotonically improves accuracy, with the verbose encoding reaching near-perfect accuracy at every $k$ studied.

\subsection{Slot Memory Recovery}
\label{sec:appendix-slot-memory}

The verbose encoding addresses scattered retrieval at substantial token cost.
At $k{=}64$ the verbose trace incurs an order-of-magnitude token blowup over the simple encoding that scales linearly with $k$ (Table~\ref{tab:slot-comparison-k64}).
A parameter-efficient alternative is \emph{slot memory}: $n$ learnable register tokens prepended to the sequence that the model can read and write throughout the forward pass, providing a parameter-level analogue of localization.
Table~\ref{tab:slot-comparison-k64} compares four routes to recovering verification accuracy at $k{=}64$.

\begin{table}[t]
\centering
\caption{Verification methods on star trees with $k{=}64$. All methods that converge reach $\ge 99.9\%$ accuracy. They differ in token overhead, convergence speed, and oracle dependency. Copy budget uses an oracle-provided scratchpad; slots learn to maintain state without oracle intervention.}
\label{tab:slot-comparison-k64}
\footnotesize
\begin{tabular}{lcccc}
  \toprule
  \textbf{Method} & \textbf{Tokens} & \textbf{Final acc.} & \textbf{Ep.\ to 99\%} & \textbf{Oracle} \\
  \midrule
  Simple (no state)         & $455$               & $\ge 99.9\%$        & $20.0 \pm 3.2$    & --- \\
  Copy $B{=}1$              & $\sim 520$          & $\ge 99.9\%$        & $8.8 \pm 1.0$     & yes \\
  Slots ($n{=}32$)          & $\mathbf{455}$      & $\mathbf{100\%}$    & $\sim 35$         & --- \\
  Verbose                   & $7{,}402$           & $100\%$             & $<10$             & --- \\
  \bottomrule
\end{tabular}
\end{table}

Slot memory recovers full accuracy at the same token cost as the simple encoding, replacing trace-level localization with a parameter-level analogue.
At $k{=}128$, however, no slot configuration recovers accuracy: every slot count we test produces near-zero all-correct verification, and a larger backbone reduces policy loss to near zero while leaving verifier loss stuck (Table~\ref{tab:slot-scaling-k128}).
The $k{=}128$ failure separates retrieval from learnability: slots address the retrieval bottleneck but not the optimization difficulty of recombining $128$ scattered features through training-loss minimization alone.

\begin{table}[t]
\centering
\caption{Slot scaling at $k{=}128$.}
\label{tab:slot-scaling-k128}
\footnotesize
\begin{tabular}{lccc}
  \toprule
  \textbf{Configuration}                & \textbf{All-correct} & \textbf{Policy loss} & \textbf{Verifier loss} \\
  \midrule
  $32$ slots, standard                  & $1.6\%$              & $1.138$              & $0.471$ \\
  $64$ slots, standard                  & $1.6\%$              & $1.141$              & $0.473$ \\
  $128$ slots, standard                 & $1.6\%$              & $1.146$              & $0.477$ \\
  $32$ slots, $d{=}512$, $6$ layers     & $1.8\%$              & $0.021$              & $0.405$ \\
  \bottomrule
\end{tabular}
\end{table}

\subsection{Connection to Backtracking Search}

The aliasing--entanglement decomposition of Appendix~\ref{sec:appendix-quantitative-theory} places scattered retrieval in the aliasing term of Eq.~\eqref{eq:aliasing-entanglement}, capturing state information missing from $T$ either because it is unrecoverable, scattered, or encoded at the wrong granularity.
History entanglement falls in the entanglement term of the same decomposition and requires an architectural rather than a trace-level fix.
\S\ref{sec:scaffold-ladder} verifies that the two requirements are independently necessary by crossing the trace-format and attention-mask axes.

\input{ext_precision_barrier}

\section{Blocks World}
\label{sec:appendix-bw}

Blocks World shows the same qualitative pattern as the constraint-satisfaction domains: SSA solve rate continues to increase with budget while causal attention plateaus.

\paragraph{Setup.}
We generate Blocks World instances with 5 or 7 blocks and a randomly sampled goal configuration.
Each decision block serializes the current block stacks, the gripper state, and the history of attempted actions.
Training follows the shared setup (\S\ref{sec:slot-memory-method}) at each scale.

\begin{table}[t]
\centering
\caption{Blocks World budget sweep on the 5-block setting (single seed).}
\label{tab:ssa-bw}
\footnotesize
\input{tab_ssa_bw}
\end{table}

\paragraph{Multi-seed comparison.}
SSA substantially outperforms causal at both 5 and 7 blocks (Table~\ref{tab:bw-7blocks}).
The gap narrows at 7 blocks, consistent with deeper planning stressing both architectures.

\begin{table}[t]
\centering
\caption{Blocks World multi-seed comparison at 5 and 7 blocks. Paired-$t$ $p$-values test the SSA-vs-Causal difference.}
\label{tab:bw-7blocks}
\footnotesize
\input{tab_bw_7blocks}
\end{table}

\section{Supplementary SSA Experiments}
\label{sec:appendix-ssa-supplementary}

\subsection{Graph Coloring: Budget Sweep and Scale-Up}
\label{sec:appendix-gc-tables}

Sweeping the inference budget on graph coloring at $n{=}30$ confirms that the SSA advantage is robust to budget choice (Table~\ref{tab:ssa-budget-sweep}).
The advantage persists from $n{=}30$ to $n{=}50$ (Table~\ref{tab:gc-scale}).

\begin{table}[t]
\centering
\caption{SSA budget sweep on graph coloring (single seed).}
\label{tab:ssa-budget-sweep}
\input{tab_ssa_budget_sweep.tex}
\end{table}

\begin{table}[t]
\centering
\caption{SSA advantage across graph coloring scales (budget 8192; $n{=}50$ uses 3 seeds).}
\label{tab:gc-scale}
\input{tab_gc_scale.tex}
\end{table}

\subsection{Mask Component Ablation}
\label{sec:ssa-ablation}

We evaluate five attention modes that independently vary problem-prefix access and trajectory blocking, all applied to the same SSA-trained weights (Table~\ref{tab:ssa-ablation}).
Blocking trajectory access has a larger effect than providing exact problem access, but both are necessary for full performance.
Random masking at matched sparsity performs poorly, which confirms the improvement is structural rather than a sparsity artifact.

\begin{table}[t]
\centering
\caption{Inference-mask sensitivity study on graph coloring.}
\label{tab:ssa-ablation}
\footnotesize
\input{tab_ssa_ablation.tex}
\end{table}

\subsection{Mask $\times$ Slots Factorial}
\label{sec:appendix-slot-mask}

The slot $\times$ mask factorial supporting the alternatives discussion in \S\ref{sec:alternatives} retrains each architecture with and without slot registers, holding the mask fixed (Table~\ref{tab:slot-mask-factorial}).
SSA's gain over causal is at least as large without slots on both SAT and graph coloring, and SSA itself improves slightly when slots are removed.
Under causal attention removing slots reduces SAT solve rate, consistent with their function as a global communication channel that the cumulative trace exploits.
The mask, not the slots, accounts for SSA's gains.

\begin{table}[t]
\centering
\caption{Mask $\times$ slots factorial. Cells report autonomous solve rate (\%, mean $\pm$ std). The slot effect is the change from removing the slot registers, holding the mask fixed; the mask effect compares SSA versus causal at fixed slot count. The SSA cells without slot registers are bolded.}
\label{tab:slot-mask-factorial}
\footnotesize
\input{tab_slot_mask_factorial}
\end{table}

\subsection{Isolating Mask from Positional Encoding}
\label{sec:ssa-position-ablation}

A crossed $2{\times}2$ ablation independently varies mask type and position scheme.
The attention mask accounts for the entire improvement, and block-relative positional reset contributes nothing (Table~\ref{tab:position-ablation}).

\begin{table}[t]
\centering
\caption{Crossed mask $\times$ position ablation on graph coloring.}
\label{tab:position-ablation}
\input{tab_position_ablation.tex}
\end{table}

\subsection{Length-OOD and Padding-Control Counterfactual}
\label{sec:appendix-padding-control}
\label{sec:length-ood}

If SSA enforces structural invariance to prior history rather than merely fitting the cumulative protocol better, it should also reduce sensitivity to the search-horizon length used during training.
We retrain both SSA and causal on cumulative traces truncated to the first $K_{\text{train}}{=}8$ decision blocks, preserving the full positional embedding table and equalizing supervised token counts; the resulting ``-short'' models contrast with ``-full'' models trained on untruncated traces as in-distribution references.
Under matched token budget and in-support accuracy, SSA substantially exceeds causal once the inference context exceeds the training-support length: SSA-short continues to improve with budget while causal-short plateaus, and the plateau is not budget exhaustion (Figure~\ref{fig:length_ood_holdout}).

\begin{figure}[t]
\centering
\includegraphics[width=\linewidth]{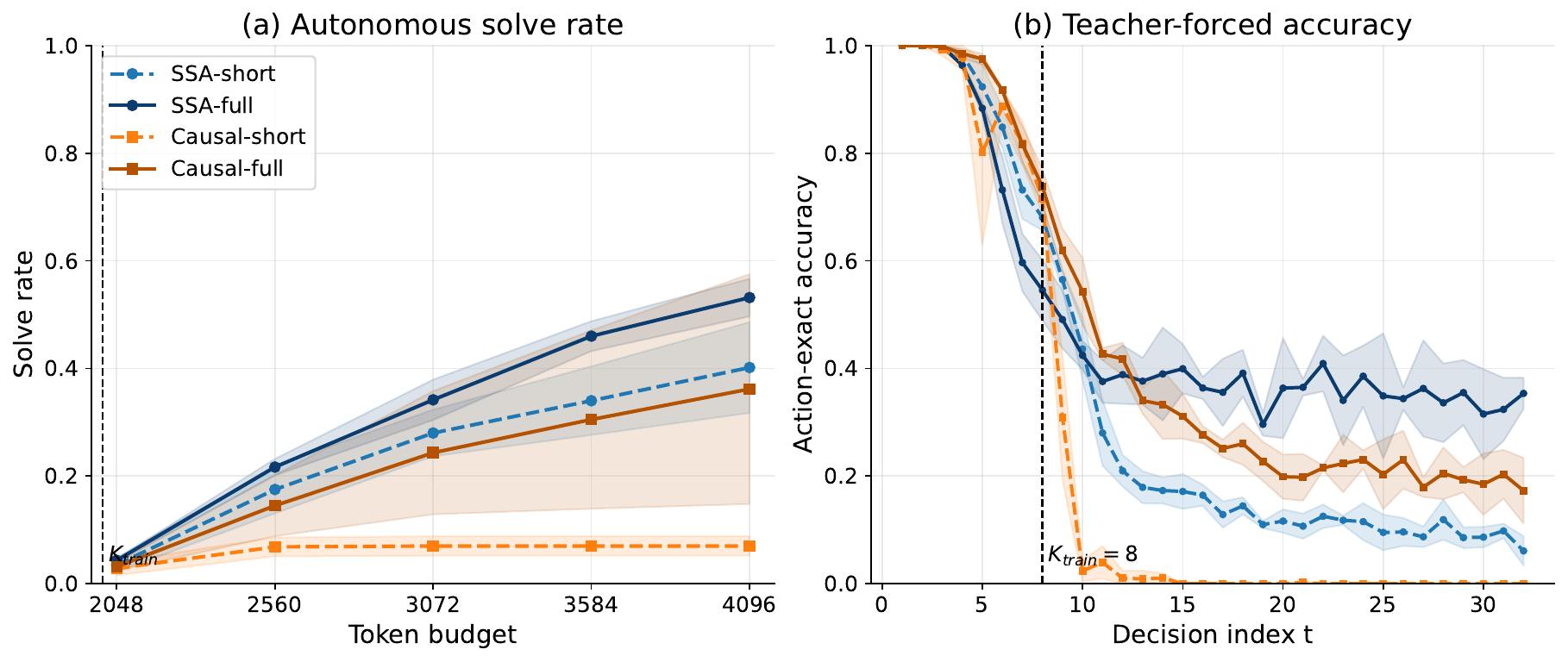}
\caption{Length-OOD comparison (3 seeds, held-out validation, shaded regions denote $\pm 1$ seed std). \textbf{(a)} Autonomous solve rate vs.\ token budget under cumulative inference. \textbf{(b)} Teacher-forced action-exact accuracy vs.\ decision depth. Dashed vertical lines mark the training-support boundary.}
\label{fig:length_ood_holdout}
\end{figure}

\paragraph{Padding-control counterfactual.}
To separate the length-OOD mechanism from positional extrapolation, we run a padding-control counterfactual on the $K_{\text{train}}{=}8$-truncated models above.
For each held-out shallow state ($t \leq K_{\text{train}}$), we insert three decision blocks from unrelated instances between the problem prefix and the current state, and measure argmax agreement against the unpadded baseline.
SSA agreement is exact across $1800$ trials under block-relative positions (the mask makes donor blocks mathematically invisible) and remains near-perfect under standard positions, with a small residual attributable to positional extrapolation rather than entanglement (Figure~\ref{fig:length_ood_padding_control}).
Causal-short's agreement is near zero regardless of position mode, because its predictions depend on prior-history content.

\begin{figure}[t]
\centering
\includegraphics[width=0.85\linewidth]{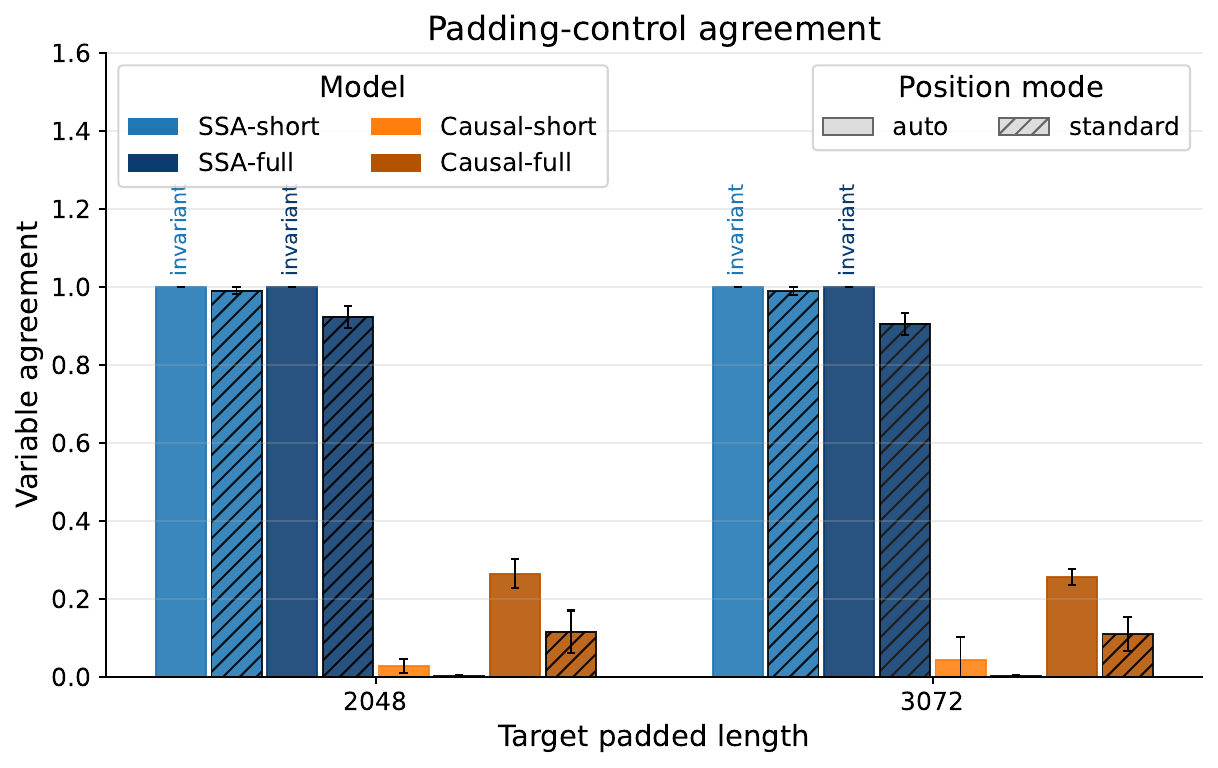}
\caption{Padding-control counterfactual (3 seeds, 600 trials per model and target length, error bars: $\pm 1$ seed std). Argmax agreement (variable choice) against the unpadded baseline.}
\label{fig:length_ood_padding_control}
\end{figure}

\subsection{Isolated Reactive Verification Experiment}
\label{sec:appendix-branch-verification}

\begin{table}[t]
\centering
\caption{Isolated reactive verification on exposed conflicts (oracle node selection).}
\label{tab:verification-experiment}
\footnotesize
\begin{tabular}{l l c c c c}
\toprule
\textbf{Domain} & \textbf{Mask} & \textbf{Accuracy $\uparrow$} & \textbf{Repeat $\downarrow$} & \textbf{False BT $\downarrow$} & \textbf{Solve \%} \\
\midrule
\multirow{2}{*}{GC ($n{=}30$)} & SSA & $\mathbf{99.5 \pm 0.0}$\% & $\mathbf{0.0}$\% & $\mathbf{0.0}$\% & $\mathbf{38.5}$\% \\
& Causal & $37.2 \pm 1.0$\% & $0.2$\% & $34.3 \pm 1.1$\% & $0.0$\% \\
\midrule
\multirow{2}{*}{SAT ($n{=}30$)} & SSA & $\mathbf{100.0 \pm 0.0}$\% & $\mathbf{0.0}$\% & $\mathbf{0.0}$\% & $\mathbf{16.5}$\% \\
& Causal & $68.3 \pm 36.3$\% & $13.3$\% & $30.7 \pm 35.9$\% & $12.2$\% \\
\bottomrule
\end{tabular}
\end{table}

Table~\ref{tab:verification-experiment} reports isolated reactive verification.
SSA achieves near-perfect accuracy on both domains across all model seeds.
Solve rate remains moderate because the experiment isolates one component (the oracle selects the branch variable, the model only decides \emph{continue} or \emph{backtrack}), and neither model can solve instances that require proactive detection of dead ends in locally consistent states.

\subsection{Learned Proactive Verifier}
\label{sec:appendix-proactive-verifier}

To probe the proactive limit referenced in \S\ref{sec:scope-and-limits}, we train two dead-end classifiers on ground-truth viability labels: a shared MLP using the same per-literal feature set as the Panel~A baseline, and a small transformer reading the canonical state block.
Both classifiers produce zero proactive backtrack interventions in the SSA search loop at default thresholds, and a threshold sweep recovers no benefit: at low thresholds the classifier never fires, and at high thresholds it fires often enough that the solve rate \emph{drops} because viable branches are removed.
Only two of five SSA seeds show any learning signal at all.
Three factors compound to make this hard: class imbalance under positive-class weighting; a train-deployment state-distribution mismatch (offline traces never visit the states an SSA solver actually reaches near the correct branch); and the available features possibly not exposing what makes a state existentially dead when SSA reaches it.
The reactive verifier described in the main text avoids all three by learning verification implicitly from the same cumulative traces used to train the branching policy.

\section{SSA Mask Variants and Positional Embeddings}
\label{sec:appendix-ssa-variants}

In addition to the selective SSA mask defined in Section~\ref{sec:ssa}, we evaluate two variants.

\paragraph{Blanket SSA.}
A \emph{blanket} SSA variant further blocks decision blocks from attending to the problem prefix.
This forces each block to rely entirely on its local state representation.

\paragraph{Sliding-window with prefix access (SWA-prefix).}
This variant is a \emph{mask topology}, distinct from the data-level sliding-window baseline of \S\ref{sec:alternatives}.
Each decision block attends to the problem prefix and to the most recent $W$ decision-block tokens, but not to earlier blocks.
SWA-prefix provides a local history window rather than full isolation.
We compare these variants in \S\ref{sec:ssa-baselines} and Appendix~\ref{sec:appendix-random-sat}.

\paragraph{Block-relative positional embeddings.}
SSA optionally resets positional indices at each block boundary.
Each decision block $B_i$ uses positions $|P|, |P|+1, \ldots, |P|+|B_i|-1$ regardless of $i$.
A crossed ablation (Appendix~\ref{sec:ssa-position-ablation}) isolates the contribution of each factor.

\section{Mask Specialization and Trace Representation}
\label{sec:appendix-mask-trace}

This section reports two diagnostic experiments that characterize the interaction between attention mask type and trace representation.

\subsection{Cross-Mask Evaluation}
\label{sec:appendix-crossmask}

We test whether the blanket-vs-selective performance difference is an inference-time masking effect or a training-time representation effect.
We evaluate models trained with one mask type using a different mask at inference, with a single seed.

\begin{table}[t]
\centering
\caption{Cross-mask evaluation (single seed).}
\label{tab:crossmask}
\footnotesize
\input{tab_crossmask}
\end{table}

Table~\ref{tab:crossmask} shows that performance degrades substantially in both directions and both domains.
This rules out an inference-time attention-dilution explanation, since models learn different internal representations under different training masks.
On SAT specifically, blanket-trained weights tolerate the selective mask better than the reverse, consistent with blanket training producing robust local-state features while selective training depends on prefix attention.

\subsection{Blanket-vs-Selective SSA Across Trace Formats}
\label{sec:appendix-factorial}

To test whether the relative advantage of blanket versus selective SSA depends on how much constraint information is serialized in each decision block, we run a controlled factorial experiment.
For GC, we generate an ``enriched'' trace variant that adds locally valid colors (\textsc{local\_legal}) to each decision block, and a ``stripped'' variant without these tokens.
For SAT, we generate a ``stripped'' variant that replaces the per-variable domain annotations (T/F/U) with a masked placeholder token (?), using the same underlying search trajectories as the enriched traces.
We train blanket and selective SSA models on each variant (5 seeds $\times$ 2 masks $\times$ 2 trace types $\times$ 2 domains $=$ 40 models) and evaluate autonomously.

\begin{table}[t]
\centering
\caption{Trace format $\times$ blanket-vs-selective SSA factorial.}
\label{tab:factorial}
\footnotesize
\input{tab_factorial}
\end{table}

Table~\ref{tab:factorial} reports the results.
The dominant finding is that trace representation matters far more than mask choice.
Removing SAT domain annotations reduces performance to near zero regardless of mask, because those annotations are structurally necessary for the agent to identify forced assignments.
The blanket-minus-selective interaction on SAT shifts in the predicted direction but is not statistically significant.
The graph-coloring interaction is null, and adding \textsc{local\_legal} tokens to graph-coloring traces does not help because most graph-coloring decision blocks already have exactly one valid color.

\subsection{Bootstrap Confidence Intervals and Per-Seed Results}
\label{sec:appendix-bootstrap}

Table~\ref{tab:bootstrap-cis} reports paired bootstrap~\citep{efron1979bootstrap} 95\% confidence intervals (10{,}000 iterations) for the key comparisons in the train-from-scratch ablation.
The SSA-versus-causal comparison excludes zero on both domains, which confirms the robustness of the primary finding.
The selective-versus-blanket comparison does not exclude zero on graph coloring and favors blanket on SAT, consistent with the analysis in Section~\ref{sec:ssa-baselines}.
The bootstrap CIs confirm that SSA outperforms the SWA-Prefix baseline on both domains.

\begin{table}[t]
\centering
\caption{Bootstrap 95\% confidence intervals for paired seed differences in the train-from-scratch ablation. CIs that exclude zero are bolded.}
\label{tab:bootstrap-cis}
\footnotesize
\input{tab_bootstrap_cis}
\end{table}

\begin{figure}[t]
    \centering
    \begin{subfigure}[t]{0.48\textwidth}
        \centering
        \includegraphics[width=\textwidth]{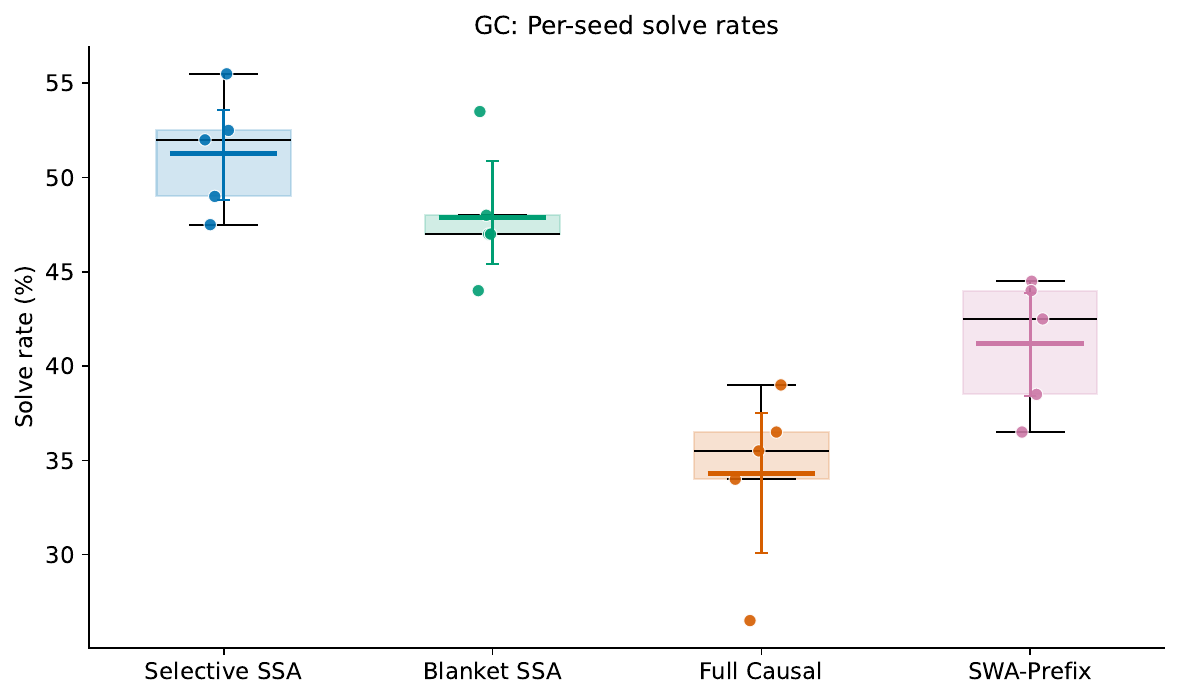}
        \caption{GC.}
        \label{fig:perseed-gc}
    \end{subfigure}
    \hfill
    \begin{subfigure}[t]{0.48\textwidth}
        \centering
        \includegraphics[width=\textwidth]{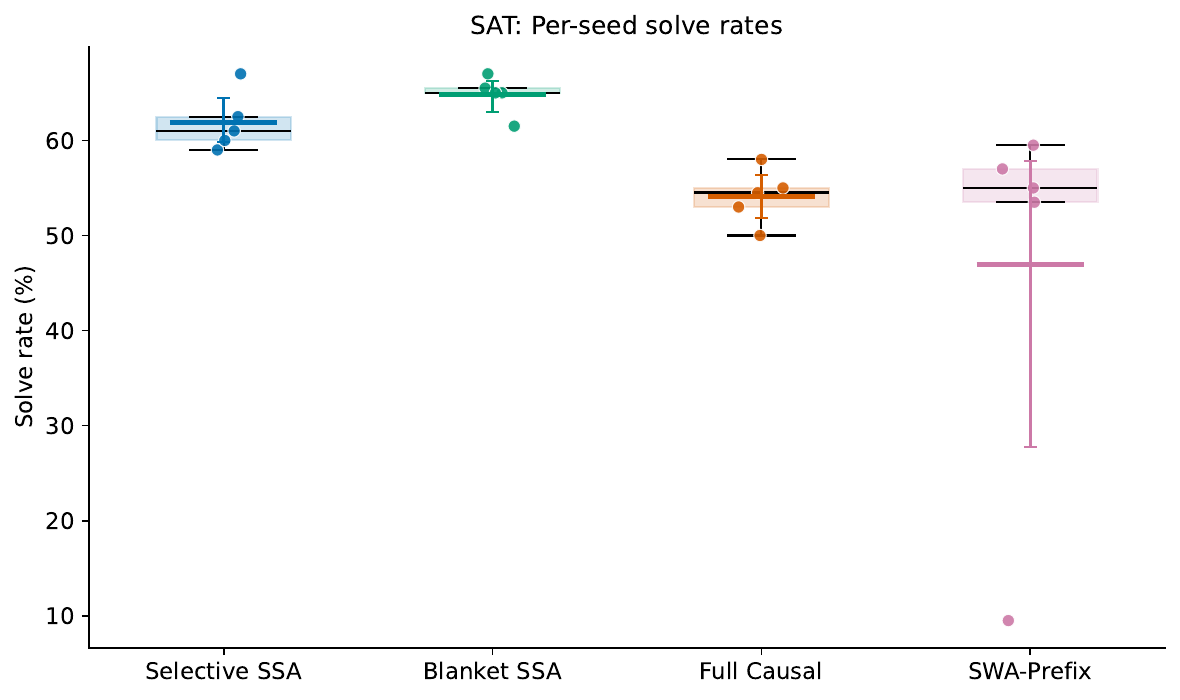}
        \caption{SAT. SWA-Prefix has high variance due to training instability on one of the seeds.}
        \label{fig:perseed-sat}
    \end{subfigure}
    \caption{Per-seed solve rates for the train-from-scratch ablation. Error bars: bootstrap 95\% CIs.}
    \label{fig:perseed}
\end{figure}

\subsection{Multi-Instance-Seed Variance}
\label{sec:appendix-multi-instance-seed}

We re-evaluate the same $5$ SSA and $5$ causal checkpoints on three additional test-instance seeds ($100$, $200$, $300$) under both inference protocols, each with $200$ planted instances ($15$ pooled runs per cell, $3{,}000$ per-instance outcomes). The state-rebuilt SSA advantage is insensitive to the test-instance seed (Table~\ref{tab:multi-instance-seed}), and cumulative-inference solve rates are moderately higher on the new seeds than on seed $42$, indicating that seed $42$ generates slightly harder instances for cumulative inference while the state-rebuilt advantage is regime-independent.

\begin{table}[t]
\centering
\caption{Multi-instance-seed robustness check on SAT. Each cell reports mean $\pm$ std across model seeds for the specified test seed.}
\label{tab:multi-instance-seed}
\footnotesize
\begin{tabular}{l l c c c c}
\toprule
\textbf{Model} & \textbf{Protocol} & \textbf{Seed $100$} & \textbf{Seed $200$} & \textbf{Seed $300$} & \textbf{Pooled ($15$ runs)} \\
\midrule
SSA    & Cumulative    & $35.4 \pm 2.7$ & $34.4 \pm 1.6$ & $37.9 \pm 2.9$ & $\mathbf{35.9 \pm 2.7}$ \\
Causal & Cumulative    & $19.8 \pm 13.7$ & $22.2 \pm 14.5$ & $23.1 \pm 14.8$ & $21.7 \pm 13.3$ \\
\midrule
SSA    & State-rebuilt & $90.9 \pm 2.4$ & $94.1 \pm 2.2$ & $91.6 \pm 2.5$ & $\mathbf{92.2 \pm 2.6}$ \\
Causal & State-rebuilt & $\phantom{0}3.6 \pm 1.1$ & $\phantom{0}4.4 \pm 1.2$ & $\phantom{0}4.8 \pm 1.3$ & $\phantom{0}4.3 \pm 1.2$ \\
\bottomrule
\end{tabular}
\end{table}

\subsection{Compute Accounting}
\label{sec:compute-accounting}

Table~\ref{tab:compute-accounting} compares per-method cost.
State-rebuilt and MLP methods process a fixed-length state representation at each decision ($L_{\mathrm{state}} \approx 1{,}300$ tokens), whereas cumulative methods process the growing trajectory ($L_{\mathrm{cum}}$ up to $8{,}192$).
Typical decision counts are $D \approx 95$ for MLP and ${\sim}30$ for SSA cumulative.
The CDCL sidecar adds two satisfiability probes per decision at ${\sim}0.1$\,ms each on an A6000 GPU, negligible relative to the neural forward-pass cost (${\sim}5$\,ms) at this scale.

\begin{table}[t]
\centering
\caption{Compute accounting per evaluation method. $D$ is the decision count, $L_{\mathrm{cum}}$ the cumulative sequence length (up to $8{,}192$), and $L_{\mathrm{state}} \approx 1{,}300$ the state-only length.}
\label{tab:compute-accounting}
\footnotesize
\input{tab_compute_accounting}
\end{table}

\section{Phase-Transition Random SAT}
\label{sec:appendix-random-sat}

To test whether the SSA advantage holds on harder instances, we train and evaluate on random 3-SAT at the phase transition~\citep{mitchell1992hard,crawford1996experimental} ($\alpha{=}4.26$, primary scale).
At this ratio, approximately half of all instances are unsatisfiable and typical solver runtime peaks.
We use only satisfiable instances for evaluation and verify them with PySAT Glucose4 before collection.

\begin{table}[t]
\centering
\caption{Planted ($\alpha=4.0$) vs.\ random phase-transition SAT ($\alpha=4.26$).}
\label{tab:random-sat}
\footnotesize
\input{tab_random_sat}
\end{table}

SSA maintains the highest solve rate and lowest variance on random SAT (Table~\ref{tab:random-sat}).
The hierarchy SSA $>$ SWA-prefix $>$ Causal is preserved.

\section{Empirical History Irrelevance and Bounded Symbolic Verifier}
\label{sec:appendix-history-irrelevance}
\label{sec:appendix-oracle-verifier}

\paragraph{Empirical history irrelevance.}
The state-equivalence framework assumes that the oracle action is independent of trajectory history given the current state ($Y \perp H \mid T$).
We give a bounded test of this assumption using logistic regression on state features $T$ alone versus $(T, H)$ jointly; a linear probe can only rule out the linearly recoverable component of conditional dependence, not the full conditional mutual information $I(Y; H \mid T)$.
Adding history features yields only a marginal lift under a linear probe (Table~\ref{tab:history-irrelevance}); the stronger behavioral evidence comes from the transplant diagnostic (\S\ref{sec:history-transplant}).

\begin{table}[t]
\centering
\caption{Empirical history-irrelevance test (500 traces each, 80/20 train/test split).}
\label{tab:history-irrelevance}
\footnotesize
\input{tab_history_irrelevance}
\end{table}

\paragraph{Bounded symbolic verifier.}
We pair each SSA checkpoint with a symbolic dead-end detection module that runs unit propagation and a bounded CDCL solver (Table~\ref{tab:bounded-verification}).
A small conflict budget captures the bulk of the improvement, indicating that most detectable dead ends are shallow.

\begin{table}[t]
\centering
\caption{Bounded symbolic verifier on SAT. ``UP'' abbreviates unit propagation.}
\label{tab:bounded-verification}
\footnotesize
\begin{tabular}{l c c c c}
\toprule
\textbf{Verif.\ budget} & \textbf{Solve \%} & \textbf{$\Delta$ vs.\ UP} & \textbf{Pol.\ forces} & \textbf{Timeout \%} \\
\midrule
0 (UP only) & $31.1 \pm 2.4$ & --- & 113 & $53.3 \pm 3.7$ \\
10 conflicts & $45.9 \pm 1.8$ & $+14.8$ pp & 276 & $48.6 \pm 3.3$ \\
50 conflicts & $\mathbf{79.1 \pm 3.0}$ & $+48.0$ pp & 493 & $20.9 \pm 3.0$ \\
100 conflicts & $80.1 \pm 3.0$ & $+49.0$ pp & 504 & $19.9 \pm 3.0$ \\
Perfect oracle & $100.0 \pm 0.0$ & $+68.9$ pp & --- & $0.0$ \\
\bottomrule
\end{tabular}
\end{table}

\section{Verifier-Only Benchmark: Protocol, Calibration, and Per-Condition Results}
\label{sec:appendix-verifier-calibration}
\label{sec:appendix-verifier-extended}

This appendix details the verifier-only state-equivalence benchmark of \S\ref{sec:history-transplant} (probe-bank construction, exact metric definitions, and full per-condition results).

\paragraph{Protocol.}
For each held-out canonical state $S$ from $200$ stochastic SAT rollouts ($300$ for graph coloring), we collect $k\!\geq\!2$ histories $H_1,\dots,H_k$ that all reach $S$.
Each model is fed $(\textsc{prefix}, H_i, S)$ in two protocols, \emph{cumulative} (the training format) and \emph{state-rebuilt} (history removed; \S\ref{sec:state-isolation}).
The binary backtrack mass at the position immediately following the state block is read as
\begin{equation}
\label{eq:p-bt}
p_{\mathrm{bt}}(M, H, S) \,=\, \frac{P_M(\textsc{conflict}\,\mid\,\textsc{prefix},\,H,\,S)}{P_M(\textsc{conflict}\,\mid\,\cdot) + \sum_{a \in \mathcal{A}_{\mathrm{cont}}(S)} P_M(a \mid \cdot)},
\end{equation}
restricted to admissible action tokens $\mathcal{A}_{\mathrm{cont}}(S)$ (selectable variables with non-empty effective domain).
The oracle viability label is $y\!=\!1$ when propagation has exposed a contradiction in $S$ and $y\!=\!0$ otherwise; this is the strict reactive scope.
The benchmark is fully model-independent: the probe bank caches $1{,}314$ canonical states for SAT $n{=}50$ ($1{,}201$ viable, $113$ exposed-conflict, $2{,}747$ transplant pairs) and $682$ for graph coloring ($427$ viable, $255$ exposed-conflict, $2{,}212$ pairs), and the same bank is reused for every checkpoint and protocol.

\paragraph{Metrics.}
Three quality metrics are reported with oracle labels: false-prune rate $\alpha_v = \Pr(p_{\mathrm{bt}}\!>\!0.5 \mid y\!=\!0)$, missed-conflict rate $\beta = \Pr(p_{\mathrm{bt}}\!\leq\!0.5 \mid y\!=\!1)$, and AUROC of $p_{\mathrm{bt}}$ against $y$.
Calibration is reported through expected calibration error (ECE) with $15$ equal-mass bins, Brier score, and area under the precision-recall curve (AUPRC).
The protocol gap $\Delta\mathrm{AUROC}$ defined in \S\ref{sec:history-transplant} is identically zero for any predictor that depends only on the canonical state (Proposition~\ref{prop:ssa-invariance}).

\begin{table}[t]
\centering
\caption{Verifier-only benchmark: full per-condition metrics across all conditions, domains, and protocols.}
\label{tab:verifier-calibration}
\footnotesize
\setlength{\tabcolsep}{3pt}
\resizebox{\textwidth}{!}{\input{tab_verifier_calibration}}
\end{table}

%% file: tab_scaffold_ladder.tex
\begin{tabular}{ll cc cc}
\toprule
& & \multicolumn{2}{c}{\textbf{Cumulative}} & \multicolumn{2}{c}{\textbf{State-rebuilt}} \\
\cmidrule(lr){3-4} \cmidrule(lr){5-6}
\textbf{State representation} & \textbf{Mask} & \textbf{Solve (\%)} & \textbf{Timeout (\%)} & \textbf{Solve (\%)} & \textbf{Timeout (\%)} \\
\midrule
\multirow{2}{*}{Enriched (UP annotations)} & SSA & $28.9 \pm 1.9$ & $70.3$ & $\mathbf{93.4 \pm 2.3}$ & $6.6$ \\
 & Causal & $16.5 \pm 11.2$ & $46.7$ & $4.8 \pm 1.2$ & $0.0$ \\
\midrule
\multirow{2}{*}{Residual-CNF (25 clauses)} & SSA & $4.7 \pm 0.5$ & $0.0$ & $4.7 \pm 0.5$ & $0.0$ \\
 & Causal & $19.8 \pm 4.0$ & $60.5$ & $4.1 \pm 1.4$ & $0.0$ \\
\midrule
\multirow{2}{*}{Stripped (variable IDs only)} & SSA & $4.7 \pm 1.3$ & $0.0$ & --- & --- \\
 & Causal & $18.2 \pm 4.0$ & $25.1$ & --- & --- \\
\bottomrule
\end{tabular}

%% file: tab_heuristic_baselines.tex
\begin{tabular}{l l c c c}
\toprule
\textbf{Method} & \textbf{State} & \multicolumn{3}{c}{\textbf{Solve \%}} \\
\cmidrule(lr){3-5}
 & & \textbf{$n{=}50$ planted} & \textbf{$n{=}50$ random} & \textbf{$n{=}75$ random} \\
 & & $\alpha{=}4.0$ & $\alpha{=}4.26$ & $\alpha{=}4.26$ \\
\midrule
\multicolumn{5}{l}{\textit{Rule-based heuristics (direct environment access, no learning)}} \\
\quad Occurrence + domain & Env & $100.0$ & $100.0$ & $100.0$ \\
\quad VSIDS + domain       & Env & $100.0$ & $100.0$ & $100.0$ \\
\quad Pure random          & Env & $100.0$ & $100.0$ & $\phantom{0}57.5$ \\
\midrule
\multicolumn{5}{l}{\textit{Neural models (enriched state tokens, trace-autonomous)}} \\
\quad MLP (state-only) & Enriched & $99.2 \pm 0.8$ & --- & --- \\
\quad SSA (state-rebuilt) & Enriched & $92.2 \pm 2.6^{\dagger}$ & --- & --- \\
\quad Causal (state-rebuilt) & Enriched & $\phantom{0}4.3 \pm 1.2^{\dagger}$ & --- & --- \\
\quad SSA (cumulative) & Enriched & $35.9 \pm 2.7^{\dagger}$ & $25.6 \pm 1.6$ & $11.4 \pm 1.4$ \\
\quad Causal (cumulative) & Enriched & $21.7 \pm 13.3^{\dagger}$ & $16.0 \pm 7.2$ & $\phantom{0}6.4 \pm 4.7$ \\
\bottomrule
\end{tabular}
{\scriptsize $^{\dagger}$Pooled over $3$ test seeds $\times$ $5$ model seeds (\S\ref{sec:appendix-multi-instance-seed}).}

%% file: tab_verifier_swap.tex
% Generated by scripts/tables/gen_tab_flagship.py
% Bolding policy: none. See the script docstring for the rationale.
\begin{tabular}{l c ccc}
\toprule
\multicolumn{5}{l}{\textbf{A.} \textit{Verification methods} (fixed SSA branching policy)} \\
\midrule
\textbf{Verifier} & \textbf{Supervision} & \multicolumn{2}{c}{\textbf{Solve \%}} & \textbf{Gap closed} \\
\cmidrule(lr){1-5}
Propagation only (no learned verifier) & --- & \multicolumn{2}{c}{$32.2 \pm 2.8$} & $0.0\%$ \\
Bounded CDCL sidecar (50) & symbolic & \multicolumn{2}{c}{$79.1 \pm 3.0$} & $69.2\%$ \\
Bounded CDCL sidecar (100) & symbolic & \multicolumn{2}{c}{$80.1 \pm 3.0$} & $70.6\%$ \\
SSA learned backtrack token & traces & \multicolumn{2}{c}{$93.4 \pm 2.3$} & $90.3\%$ \\
Perfect dead-end oracle & perfect & \multicolumn{2}{c}{$100.0 \pm 0.0$} & $100.0\%$ \\
\midrule
\multicolumn{5}{l}{\textbf{B.} \textit{Oracle decomposition}} \\
\midrule
& & \textbf{No oracle} & \textbf{+ Perfect oracle} & \textbf{Lift (pp)} \\
\cmidrule(lr){3-5}
SSA & & $28.9 \pm 1.9$ & $100.0 \pm 0.0$ & $+71.1$ \\
Causal & & $16.5 \pm 11.2$ & $46.3 \pm 38.0$ & $+29.8$ \\
\bottomrule
\end{tabular}

%% file: tab_tokenization_examples.tex
% Auto-generated tokenization examples.
% Do not edit by hand; rerun the generator to regenerate.

\paragraph{SAT, satisfying trace ($n=3$, $4$ clauses).}
\begin{footnotesize}
\begin{verbatim}
# Problem prefix P
[BOS] [CLAUSES] C0 : +v0 -v1 +v2 SEP C1 : -v0 +v1 +v2 SEP C2 : +v0 +v1 -v2
SEP C3 : -v0 -v1 +v2 SEP [SEARCH]

# Decision block B1
STATE v0 U v1 U v2 U SEP [PROP] C0 : +v0 -v1 +v2 SEP +v0 U -v1 U +v2 U SEP
SAT_OK [/PROP] v0 T OK v0 T

# Decision block B2
STATE v1 U v2 U SEP [PROP] C1 : -v0 +v1 +v2 SEP -v0 F +v1 U +v2 U SEP
SAT_OK [/PROP] v1 T OK v1 T

# Decision block B3
STATE v2 U SEP [PROP] C3 : -v0 -v1 +v2 SEP -v0 F -v1 F +v2 U SEP UNIT
[/PROP] v2 T OK v2 T SOLVED [EOS]
\end{verbatim}
\end{footnotesize}

\paragraph{SAT, trace ending in a conflict ($n=3$, $6$ clauses).}
\begin{footnotesize}
\begin{verbatim}
# Problem prefix P
[BOS] [CLAUSES] C0 : +v0 +v1 +v2 SEP C1 : +v0 +v1 -v2 SEP C2 : +v0 -v1 +v2
SEP C3 : +v0 -v1 -v2 SEP C4 : -v0 +v1 +v2 SEP C5 : -v0 +v1 -v2 SEP [SEARCH]

# Decision block B1
STATE v0 U v1 U v2 U SEP [PROP] C0 : +v0 +v1 +v2 SEP +v0 U +v1 U +v2 U SEP
SAT_OK [/PROP] v2 F OK v2 F

# Decision block B2
STATE v0 U v1 U SEP [PROP] C0 : +v0 +v1 +v2 SEP +v0 U +v1 U +v2 F SEP
SAT_OK [/PROP] v0 F OK v0 F

# Decision block B3
STATE v0 U v1 U SEP [PROP] C0 : +v0 +v1 +v2 SEP +v0 F +v1 U +v2 F SEP UNIT
[/PROP] CONFLICT C5 BJ L0 [EOS]
\end{verbatim}
\end{footnotesize}

\paragraph{Graph coloring, satisfying trace ($n=4$, $4$-cycle).}
\begin{footnotesize}
\begin{verbatim}
# Problem prefix P
[BOS] [GRAPH] N0 : N1 N3 SEP N1 : N0 N2 SEP N2 : N1 N3 SEP N3 : N0 N2 SEP
[SEARCH]

# Decision block B1
STATE DS4 N0 DS4 N1 DS4 N2 DS4 N3 SEP N0 M1111 C1 OK

# Decision block B2
STATE DS3 N1 DS4 N2 DS3 N3 SEP N1 M0111 C2 OK SOLVED [EOS]
\end{verbatim}
\end{footnotesize}

%% file: ext_precision_barrier.tex
\section{The False-Prune Threshold} \label{sec:precision-barrier}

This appendix derives the false-prune threshold that any learned dead-end verifier must clear in backtracking search.
The threshold is the precision target $(1-\alpha_v)^M$ that the operating-point analysis in \S\ref{sec:verification-experiment} compares SSA against, where $\alpha_v$ is the conditional false-prune rate on viable queried states and $M$ is the number of branch points on a solution path.
All experiments use the primary 3-SAT setting (\S\ref{sec:slot-memory-method}).

\subsection{FP/FN Asymmetry}
\label{sec:fp-fn-asymmetry}

Holding the branching policy fixed and replacing reactive-only backtracking with a perfect dead-end oracle closes essentially the entire remaining solve-rate gap, and the gap is attributable to eliminating false-prune decisions.
To separate false-positive (FP) and false-negative (FN) damage, we replace the neural verifier with a PySAT oracle and inject controlled noise: with probability $p_{\text{fp}}$ the oracle flips a viable state to dead, and with probability $p_{\text{fn}}$ flips a dead-end state to viable.

\input{tab_corruption_curve}

\begin{figure}[t]
\centering
\includegraphics[width=0.65\textwidth]{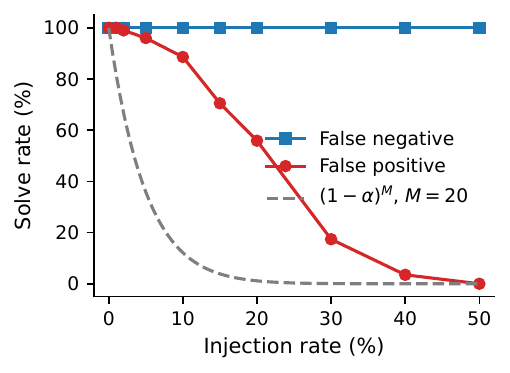}
\caption{Solve rate under synthetic corruption of oracle dead-end decisions. False positives degrade solve rate exponentially while false negatives have no effect, and the dashed curve is the per-path survival prediction $(1-\alpha_v)^M$ with $M{=}20$ branch points.}
\label{fig:fp-fn-asymmetry}
\end{figure}

FP corruption degrades solve rate rapidly while FN corruption has no effect: a false positive permanently removes a viable search branch, while a false negative only delays dead-end detection by additional search steps because the solver falls back to reactive backtracking.
The conditional false-prune rate $\alpha_v$ governs the corruption curve more cleanly than standard precision (whose value the viable-class prevalence $\rho$ and missed-conflict rate $\beta$ confound); plotting against $\alpha_v$ produces a monotonic transition across SAT $n{=}50$, SAT $n{=}100$, and GC $n{=}30$ (Figure~\ref{fig:alpha-v-cliff}).
A multi-path model $1 - (1 - (1-\alpha_v)^{M_\text{eff}})^{R_\text{eff}}$ fits the corruption curve well; the per-path survival model $(1-\alpha_v)^{M}$ from Proposition~\ref{prop:fp-compounding} is a strict lower bound.
The threshold is robust across structured corruption modes (depth-dependent, confidence-dependent, clustered), training-seed variability, and non-planted random SAT at the phase transition.

\begin{figure}[t]
    \centering
    \includegraphics[width=\textwidth]{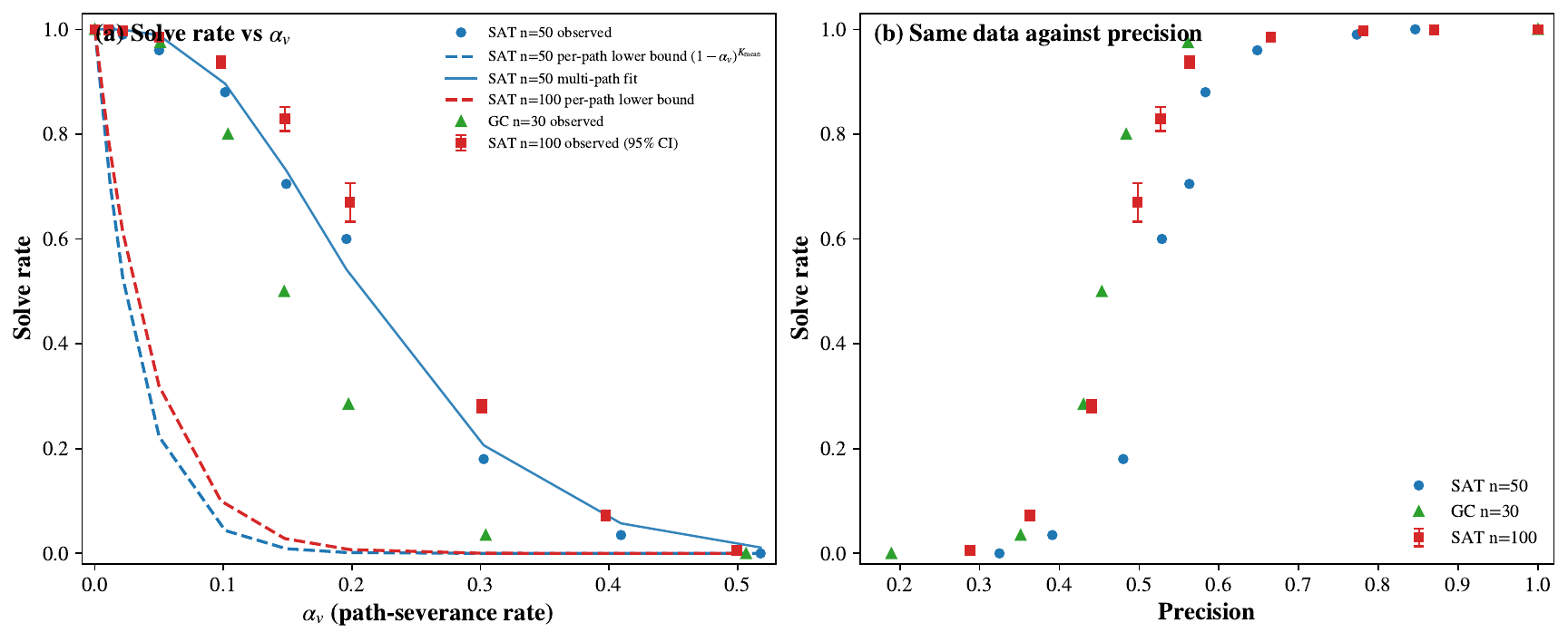}
    \caption{Solve rate versus conditional false-prune rate $\alpha_v$ (left) and standard precision (right), for SAT $n{=}50$, SAT $n{=}100$, and GC $n{=}30$. The solid line is the SAT $n{=}50$ multi-path fit ($R^2 = 0.996$); the dashed curves are the per-path lower bounds $(1-\alpha_v)^{M}$ from Proposition~\ref{prop:fp-compounding}, with $M$ the mean number of branch points per path.}
    \label{fig:alpha-v-cliff}
\end{figure}

\subsection{Formal Analysis}
\label{sec:formal-cliff}

We now provide a formal model that explains the FP/FN asymmetry, predicts the threshold location, and characterizes its dependence on problem size.
Let $\Gamma$ be a finite rooted search tree, $u \in \Gamma$ a partial state, and call $u$ \emph{viable} if it has at least one satisfying descendant and \emph{dead} otherwise.
A traversal policy $\pi$ is \emph{exhaustive} if it eventually explores every child of every visited node.
The solver queries a verifier $V$ at visited nodes; a \emph{false positive} is $\mathrm{back}$ on viable, a \emph{false negative} is $\mathrm{cont}$ on dead.
Within this appendix $\alpha_j$, $\alpha_c$, and $\bar{\alpha}_M$ denote per-query, critical, and mean false-prune rates respectively (instances of the conditional false-prune rate $\alpha_v$ from \S\ref{sec:verification-bottleneck}).

\begin{proposition}[False-Positive Compounding]
\label{prop:fp-compounding}
Consider a run with exhaustive policy $\pi$.
Suppose that, absent any verifier FP, the run reaches a satisfying leaf along a solution path $\mathcal{P} = (u_0, u_1, \ldots, u_\ell)$. The calligraphic $\mathcal{P}$ avoids a clash with the problem prefix $P$ of \S\ref{sec:ssa}.
Let $u_{t_1}, \ldots, u_{t_M}$ be the viable nodes on $\mathcal{P}$ where the solver queries the verifier.
For $j = 1, \ldots, M$, let
\[
\alpha_j := \Pr(\text{FP at } u_{t_j} \mid \text{no FP at } u_{t_1}, \ldots, u_{t_{j-1}}).
\]
Then the probability that the verifier never severs $\mathcal{P}$ is
\begin{equation}
\label{eq:fp-compounding}
\Pr(\mathcal{P} \text{ survives}) = \prod_{j=1}^{M} (1 - \alpha_j).
\end{equation}
Under homogeneous FP rate $\alpha_j \equiv \alpha_v$,
\begin{equation}
\label{eq:fp-homogeneous}
\Pr(\mathcal{P} \text{ survives}) = (1 - \alpha_v)^M = \exp(-\alpha_v M + O(\alpha_v^2 M)).
\end{equation}
Under heterogeneous rates with mean $\bar{\alpha}_M = M^{-1} \sum_{j=1}^{M} \alpha_j$,
\begin{equation}
\label{eq:fp-amgm}
\Pr(\mathcal{P} \text{ survives}) = \prod_{j=1}^{M} (1 - \alpha_j) \le (1 - \bar{\alpha}_M)^M.
\end{equation}
\end{proposition}

\begin{proof}
By the chain rule, $\Pr(\mathcal{P} \text{ survives}) = \prod_{j=1}^{M} (1 - \alpha_j)$.
The homogeneous case follows from $\log(1-\alpha_v) = -\alpha_v + O(\alpha_v^2)$.
The heterogeneous bound follows from AM--GM.
\end{proof}

\begin{proposition}[False-Negative Bounded Overhead]
\label{prop:fn-benignity}
Let $V$ be any verifier with no false positives but arbitrary false negatives, and let $\pi$ be exhaustive.
If the budget $B$ is unlimited, $V$ is complete and every viable solution path remains present.
Let $F$ be the set of minimal dead nodes where $V^\star$ would backtrack but $V$ continues, and let $\Gamma_u$ be the dead subtree rooted at $u$.
The additional backtracks relative to an oracle verifier $V^\star$ are bounded by $\sum_{u \in F} (|\Gamma_u| - 1)$.
\end{proposition}

\begin{proof}
Since $V$ has no false positives, every viable path remains present.
Exhaustiveness ensures eventual termination.
The subtrees $\Gamma_u$ are pairwise disjoint by minimality, and each contributes at most $|\Gamma_u|-1$ additional backtracks.
\end{proof}

FP damage is multiplicative, since survival factors $(1-\alpha_j) = (1-\alpha_v)$ in the homogeneous case accumulate exponentially.
FN damage is additive, bounded by the total size of dead subtrees.
This contrast is the formal source of the sharp false-prune threshold.

\begin{proposition}[Critical Threshold and Transition Width]
\label{prop:cliff-threshold}
Under the homogeneous model with $m$ queries on viable nodes and FP rate $\alpha_v$, the path-survival probability is $s_m(\alpha_v) = (1-\alpha_v)^m$.
For target level $q$, the critical rate is $\alpha_c(q; m) = 1 - q^{1/m} = \log(1/q)/m + O(m^{-2})$, with transition width $\alpha_c(q_2; m) - \alpha_c(q_1; m) = \log(q_1/q_2)/m + O(m^{-2})$ and local slope $|s_m'(\alpha_c)| = \Theta(m)$.
The tolerable rate shrinks as $1/m$ and the slope grows as $m$, so the threshold sharpens with more queries.
\end{proposition}

\begin{proof}
Solving $(1-\alpha_v)^m = q$ gives $\alpha_c = 1 - q^{1/m}$, and Taylor expansion and differentiation give the remaining claims.
\end{proof}

As instance size $n$ grows with $m_n \to \infty$, the critical rate shrinks as $\alpha_c \sim \log(1/q)/m_n$, explaining why the threshold sharpens across the graph-coloring and SAT settings of the main body and persists at $n{=}100$ (\S\ref{sec:n100-scale}).

\subsection{Learned Verifier and $n{=}100$ Scaling}
\label{sec:threshold-operating}
\label{sec:n100-scale}

Sweeping the look-ahead threshold $\tau$ on a learned MLP verifier produces a sharp, asymmetric operating curve: conservative thresholds degrade gradually, aggressive thresholds drop solve rate to near zero, and a PySAT oracle achieves perfect solve rate (Table~\ref{tab:oracle_upperbound}).
Off-policy training fails to reach the required precision even after scaling the training data by more than an order of magnitude because the training distribution contains mostly easy shallow dead-ends while deployment encounters hard dead-ends near critical decision points; on-policy hard-negative mining addresses this distribution shift but drifts across rounds and drops to near-zero solve rate within a small number of rounds.
The FP/FN asymmetry and oracle upper bound persist at $n{=}100$, with the threshold tightening as $\alpha_c \propto 1/m$ (Proposition~\ref{prop:cliff-threshold}).

\input{tab_oracle_upperbound}

%% file: tab_corruption_curve.tex
\begin{table}[t]
\centering
\caption{Synthetic corruption of oracle look-ahead decisions; some rows reported single-seed where the multi-seed run is unavailable.}
\label{tab:corruption_curve}
\scriptsize
\setlength{\tabcolsep}{4pt}
\begin{tabular}{ccccccc}
  \toprule
  \textbf{FP Rate} & \textbf{Solve (FP)} & \textbf{$\alpha_v$} & \textbf{Precision} & \textbf{FN Rate} & \textbf{Solve (FN)} & \textbf{Recall} \\
  \midrule
  \textbf{0.0\%} & \textbf{100.0\%} & \textbf{0.000} & \textbf{1.000} & \textbf{0.0\%} & \textbf{100.0\%} & \textbf{1.000} \\
  1.0\% & 100.0\% & 0.012 & 0.846 & 1.0\% & 100.0\% & 0.992 \\
  2.0\% & 99.0\% & 0.022 & 0.773 & 2.0\% & 100.0\% & 0.981 \\
  5.0\% & 96.0\% & 0.050 & 0.648 & 5.0\% & 100.0\% & 0.949 \\
  10.0\% & $88.6 \pm 1.3\%$ & 0.101 & 0.583 & 10.0\% & 100.0\% & 0.910 \\
  15.0\% & 70.5\% & 0.149 & 0.563 & 15.0\% & 100.0\% & 0.868 \\
  20.0\% & $55.9 \pm 2.8\%$ & 0.196 & 0.529 & 20.0\% & 100.0\% & 0.818 \\
  30.0\% & $17.4 \pm 2.6\%$ & 0.303 & 0.480 & 30.0\% & 100.0\% & 0.711 \\
  40.0\% & 3.5\% & 0.409 & 0.391 & 40.0\% & 100.0\% & 0.589 \\
  50.0\% & 0.0\% & 0.518 & 0.325 & 50.0\% & 100.0\% & 0.480 \\
  \bottomrule
\end{tabular}
\end{table}

%% file: tab_oracle_upperbound.tex
\begin{table}[t]
\centering
\caption{Oracle upper bound.}
\label{tab:oracle_upperbound}
\footnotesize
\begin{tabular}{lccc}
  \toprule
  \textbf{Configuration} & \textbf{Solve Rate} & \textbf{BTs/Solve} & \textbf{Proactive BTs} \\
  \midrule
  Reactive only & $50.8 \pm 2.0\%$ & 518.3 & 0 \\
  \textbf{Oracle look-ahead} & \textbf{$100.0 \pm 0.0\%$} & \textbf{1.9} & \textbf{358} \\
  \bottomrule
\end{tabular}
\end{table}

%% file: tab_ssa_bw.tex
\begin{tabular}{ll ccc}
\toprule
\textbf{Mode} & \textbf{Budget} & \textbf{Solve Rate} & \textbf{Repeat Rate} & \textbf{Mean BT} \\
\midrule
SSA & 2048 & 0.38 & 0.04 & 15.1 \\
SSA & 4096 & 0.53 & 0.04 & 29.7 \\
\midrule
Causal & 2048 & 0.12 & 0.33 & 0.0 \\
Causal & 4096 & 0.12 & 0.33 & 0.0 \\
\bottomrule
\end{tabular}

%% file: tab_bw_7blocks.tex
\begin{tabular}{l cc cc}
\toprule
& \multicolumn{2}{c}{\textbf{5 blocks}} & \multicolumn{2}{c}{\textbf{7 blocks}} \\
\cmidrule(lr){2-3} \cmidrule(lr){4-5}
\textbf{Method} & \textbf{Solve} & \textbf{$p$} & \textbf{Solve} & \textbf{$p$} \\
\midrule
Causal & $8.7 \pm 2.4$ & & $5.2 \pm 1.9$ & \\
\textbf{SSA (ours)} & $\mathbf{50.7 \pm 2.9}$ & $2.1{\times}10^{-5}$ & $\mathbf{16.1 \pm 3.4}$ & $1.3{\times}10^{-4}$ \\
\bottomrule
\end{tabular}

%% file: tab_ssa_budget_sweep.tex
\begin{tabular}{ll ccc}
\toprule
\textbf{Mode} & \textbf{Budget} & \textbf{Solve Rate} & \textbf{Repeat Rate} & \textbf{Mean BT} \\
\midrule
SSA & 2048 & 0.695 & 0.00 & 10.8 \\
SSA & 4096 & 0.90 & 0.00 & 16.3 \\
SSA & 8192 & 0.97 & 0.00 & 19.2 \\
\midrule
Causal & 2048 & 0.685 & 0.08 & 10.4 \\
Causal & 4096 & 0.685 & 0.08 & 10.4 \\
Causal & 8192 & 0.685 & 0.08 & 10.4 \\
\bottomrule
\end{tabular}

%% file: tab_gc_scale.tex
\begin{tabular}{l cc cc}
\toprule
& \multicolumn{2}{c}{\textbf{$n{=}30$}} & \multicolumn{2}{c}{\textbf{$n{=}50$} (3 seeds)} \\
\cmidrule(lr){2-3} \cmidrule(lr){4-5}
\textbf{Method} & \textbf{Solve (\%)} & \textbf{Repeat} & \textbf{Solve (\%)} & \textbf{Repeat} \\
\midrule
Causal & $69.9 \pm 1.8$ & 0.08 & $0.2 \pm 0.3$ & 0.00 \\
\textbf{SSA (ours)} & $\mathbf{98.3 \pm 0.6}$ & \textbf{0.00} & $\mathbf{96.2 \pm 1.5}$ & \textbf{0.00} \\
\midrule
$\Delta$ & $+28.4$ & & $+96.0$ & \\
$p$ (paired $t$) & $1.4{\times}10^{-6}$ & & $1.2{\times}10^{-4}$ & \\
\bottomrule
\end{tabular}

%% file: tab_ssa_ablation.tex
\begin{tabular}{lcc ccc}
\toprule
\textbf{Mode} & \textbf{Problem} & \textbf{Traj.} & \textbf{Solve} & \textbf{Repeat} & \textbf{Mean BT} \\
\midrule
Selective SSA & Exact & Blocked & 0.695 & 0.00 & 10.8 \\
Blanket SSA & Blocked & Blocked & 0.45 & 0.07 & 15.2 \\
Full causal & Exact & Exact & 0.16 & 0.02 & 22.3 \\
Reverse selective & Blocked & Exact & 0.09 & 0.07 & 24.1 \\
Random matched & Random & Random & 0.05 & 0.02 & 25.1 \\
\bottomrule
\end{tabular}

%% file: tab_slot_mask_factorial.tex
% Generated by scripts/tables/gen_tab_slot_mask_factorial.py
\begin{tabular}{l c c c c}
\toprule
 & \multicolumn{2}{c}{\textbf{SAT $n{=}50$}} & \multicolumn{2}{c}{\textbf{GC $n{=}30$}} \\
\cmidrule(lr){2-3} \cmidrule(lr){4-5}
\textbf{Mask} & \textbf{$32$ slots} & \textbf{$0$ slots} & \textbf{$32$ slots} & \textbf{$0$ slots} \\
\midrule
SSA & $28.9 \pm 1.9$ & $\mathbf{34.1 \pm 2.0}$ & $51.3 \pm 3.1$ & $\mathbf{57.5 \pm 7.7}$ \\
Causal & $16.5 \pm 11.2$ & $8.2 \pm 10.6$ & $34.3 \pm 4.7$ & $34.2 \pm 4.8$ \\
\midrule
Slot effect (SSA) & \multicolumn{2}{c}{$+5.2$ pp} & \multicolumn{2}{c}{$+6.2$ pp} \\
Slot effect (Causal) & \multicolumn{2}{c}{$-8.3$ pp} & \multicolumn{2}{c}{$-0.1$ pp} \\
Mask effect & $+12.4$ pp & $+25.9$ pp & $+17.0$ pp & $+23.3$ pp \\
\bottomrule
\end{tabular}

%% file: tab_position_ablation.tex
\begin{tabular}{ll cc}
\toprule
\textbf{Mask} & \textbf{Positions} & \textbf{Solve Rate} & \textbf{Repeat Rate} \\
\midrule
Causal & Standard & 0.160 & 0.02 \\
Causal & Block-relative & 0.055 & 0.09 \\
SSA & Standard & 0.725 & 0.00 \\
SSA & Block-relative & 0.695 & 0.00 \\
\bottomrule
\end{tabular}

%% file: tab_crossmask.tex
\begin{tabular}{llccc}
\toprule
\textbf{Domain} & \textbf{Condition} & \textbf{Train Mask} & \textbf{Eval Mask} & \textbf{Solve \%} \\
\midrule
\multirow{4}{*}{SAT} & Matched & Blanket & Blanket & 61.5 \\
& Matched & Selective & Selective & 59.0 \\
& Cross & Blanket & Selective & 27.0 \\
& Cross & Selective & Blanket & 10.5 \\
\midrule
\multirow{4}{*}{GC} & Matched & Blanket & Blanket & 47.0 \\
& Matched & Selective & Selective & 47.5 \\
& Cross & Blanket & Selective & 4.5 \\
& Cross & Selective & Blanket & 4.0 \\
\bottomrule
\end{tabular}

%% file: tab_factorial.tex
\begin{tabular}{llccc}
\toprule
\textbf{Domain} & \textbf{Trace} & \textbf{Mask} & \textbf{Solve \%} & \textbf{$\Delta$ (Blanket$-$Selective)} \\
\midrule
\multirow{4}{*}{GC} & \multirow{2}{*}{Enriched} & Blanket & $49.5 \pm 4.6$ & \multirow{2}{*}{$+0.7$} \\
& & Selective & $48.8 \pm 4.3$ & \\
& \multirow{2}{*}{Stripped} & Blanket & $53.4 \pm 2.4$ & \multirow{2}{*}{$+0.4$} \\
& & Selective & $53.0 \pm 1.2$ & \\
\midrule
\multirow{4}{*}{SAT} & \multirow{2}{*}{Enriched} & Blanket & $64.8 \pm 2.0$ & \multirow{2}{*}{$+2.9$} \\
& & Selective & $61.9 \pm 3.1$ & \\
& \multirow{2}{*}{Stripped} & Blanket & $13.1 \pm 2.5$ & \multirow{2}{*}{$-0.8$} \\
& & Selective & $13.9 \pm 1.5$ & \\
\midrule
\multicolumn{4}{l}{\textit{Interaction ($\Delta_\text{enriched} - \Delta_\text{stripped}$)}} \\
\quad GC & & & $+0.3$ ($p = 0.925$) \\
\quad SAT & & & $+3.7$ ($p = 0.139$) \\
\bottomrule
\end{tabular}

%% file: tab_bootstrap_cis.tex
\begin{tabular}{llcc}
\toprule
\textbf{Domain} & \textbf{Comparison} & \textbf{Mean $\Delta$} & \textbf{95\% Bootstrap CI} \\
\midrule
GC & SSA $-$ Causal & $+17.0$ & $[12.9, 21.6]$ \\
GC & Selective $-$ Blanket & $+3.4$ & $[0.0, 7.5]$ \\
GC & SSA $-$ SWA$-$Prefix & $+10.1$ & $[7.8, 13.2]$ \\
SAT & SSA $-$ Causal & $+10.7$ & $[7.5, 14.1]$ \\
SAT & Selective $-$ Blanket & $-2.9$ & $[-4.4, -1.3]$ \\
SAT & SSA $-$ SWA$-$Prefix & $+17.9$ & $[8.3, 35.0]$ \\
\bottomrule
\end{tabular}

%% file: tab_compute_accounting.tex
\begin{tabular}{lccccc}
\toprule
\textbf{Method} & \textbf{Fwd.} & \textbf{Tok./pass} & \textbf{Total tok.} & \textbf{Rebuilds} & \textbf{Verifier} \\
\midrule
Causal (cumulative) & $D$ & ${\sim}L_{\mathrm{cum}}$ & $D \cdot L_{\mathrm{cum}}$ & $0$ & $0$ \\
SSA (cumulative) & $D$ & ${\sim}L_{\mathrm{cum}}$ & $D \cdot L_{\mathrm{cum}}$ & $0$ & $0$ \\
Context-clearing & $D$ & ${\sim}L_{\mathrm{state}}$ & $D \cdot L_{\mathrm{state}}$ & $D$ & $0$ \\
MLP / Linear & $D$ & $1011$ & $D \cdot 1011$ & $D$ & $0$ \\
SSA + CDCL($c$) & $D$ & ${\sim}L_{\mathrm{cum}}$ & $D \cdot L_{\mathrm{cum}}$ & $0$ & $2D$ \\
\bottomrule
\end{tabular}

%% file: tab_random_sat.tex
\begin{tabular}{lccc}
\toprule
\textbf{Mask} & \textbf{Planted ($\alpha=4.0$)} & \textbf{Random ($\alpha=4.26$)} & \textbf{$\Delta$} \\
\midrule
SSA & $28.9 \pm 1.9$ & $\mathbf{25.6 \pm 1.6}$ & $-3.3$ \\
SWA-prefix & $28.2 \pm 1.8$ & $21.4 \pm 2.8$ & $-6.8$ \\
Causal (full) & $16.5 \pm 11.2$ & $16.0 \pm 7.2$ & $-0.5$ \\
\bottomrule
\end{tabular}

%% file: tab_history_irrelevance.tex
\begin{tabular}{lc}
\toprule
\textbf{Measurement} & \textbf{Mean $\pm$ std} \\
\midrule
\multicolumn{2}{l}{\textit{History-irrelevance gap (logistic regression, enriched features)}} \\
Accuracy on $T$ only & $81.6 \pm 0.1\%$ \\
Accuracy on $(T, H)$ & $83.0 \pm 0.4\%$ \\
$\Delta$ (history lift) & $+1.4 \pm 0.3$ pp \\
\midrule
\multicolumn{2}{l}{\textit{Linear probes on frozen representations ($d{=}256$)}} \\
SSA probe $\to$ oracle label & $89.6 \pm 0.4\%$ \\
Causal probe $\to$ oracle label & $92.4 \pm 0.2\%$ \\
\bottomrule
\end{tabular}

%% file: tab_verifier_calibration.tex
\begin{tabular}{lllcccccc}
\toprule
Architecture & Domain & Protocol & $\alpha_v \pm s$ & $\beta \pm s$ & AUROC $\pm s$ & AUPRC $\pm s$ & ECE & Brier \\
\midrule
SSA (selective) & SAT $n=50$ & cumulative & $0.092 \pm 0.011$ & $0.536 \pm 0.237$ & $0.904 \pm 0.020$ & $0.392 \pm 0.098$ & $0.104 \pm 0.020$ & $0.110 \pm 0.019$ \\
SSA (selective) & SAT $n=50$ & state-rebuilt & $0.092 \pm 0.011$ & $0.536 \pm 0.237$ & $0.904 \pm 0.020$ & $0.392 \pm 0.098$ & $0.104 \pm 0.020$ & $0.110 \pm 0.019$ \\
SSA (selective) & GC $n=30$ & cumulative & $0.000 \pm 0.000$ & $0.965 \pm 0.029$ & $1.000 \pm 0.000$ & $1.000 \pm 0.000$ & $0.337 \pm 0.024$ & $0.312 \pm 0.041$ \\
SSA (selective) & GC $n=30$ & state-rebuilt & $0.000 \pm 0.000$ & $0.968 \pm 0.024$ & $1.000 \pm 0.000$ & $1.000 \pm 0.001$ & $0.347 \pm 0.014$ & $0.329 \pm 0.024$ \\
Causal & SAT $n=50$ & cumulative & $0.029 \pm 0.010$ & $0.510 \pm 0.064$ & $0.944 \pm 0.007$ & $0.539 \pm 0.038$ & $0.041 \pm 0.007$ & $0.057 \pm 0.005$ \\
Causal & SAT $n=50$ & state-rebuilt & $0.027 \pm 0.029$ & $0.986 \pm 0.027$ & $0.772 \pm 0.096$ & $0.188 \pm 0.039$ & $0.091 \pm 0.013$ & $0.094 \pm 0.012$ \\
Causal & GC $n=30$ & cumulative & $0.000 \pm 0.000$ & $1.000 \pm 0.000$ & $0.987 \pm 0.011$ & $0.963 \pm 0.028$ & $0.371 \pm 0.001$ & $0.368 \pm 0.003$ \\
Causal & GC $n=30$ & state-rebuilt & $0.047 \pm 0.028$ & $0.984 \pm 0.016$ & $0.880 \pm 0.008$ & $0.650 \pm 0.026$ & $0.376 \pm 0.026$ & $0.377 \pm 0.016$ \\
Current-block-only & SAT $n=50$ & cumulative & $0.122 \pm 0.008$ & $0.421 \pm 0.120$ & $0.807 \pm 0.025$ & $0.228 \pm 0.020$ & $0.180 \pm 0.041$ & $0.125 \pm 0.020$ \\
Current-block-only & SAT $n=50$ & state-rebuilt & $0.122 \pm 0.008$ & $0.421 \pm 0.120$ & $0.807 \pm 0.025$ & $0.228 \pm 0.020$ & $0.180 \pm 0.041$ & $0.125 \pm 0.020$ \\
Block dropout $p=0.5$ & SAT $n=50$ & cumulative & $0.029 \pm 0.006$ & $0.515 \pm 0.114$ & $0.942 \pm 0.015$ & $0.548 \pm 0.075$ & $0.038 \pm 0.007$ & $0.056 \pm 0.008$ \\
Block dropout $p=0.5$ & SAT $n=50$ & state-rebuilt & $0.015 \pm 0.014$ & $0.993 \pm 0.012$ & $0.735 \pm 0.084$ & $0.171 \pm 0.044$ & $0.090 \pm 0.014$ & $0.090 \pm 0.006$ \\
LSTM & SAT $n=50$ & cumulative & $0.050 \pm 0.025$ & $0.288 \pm 0.068$ & $0.938 \pm 0.020$ & $0.604 \pm 0.135$ & $0.048 \pm 0.020$ & $0.056 \pm 0.018$ \\
LSTM & SAT $n=50$ & state-rebuilt & $0.089 \pm 0.014$ & $0.851 \pm 0.196$ & $0.832 \pm 0.059$ & $0.237 \pm 0.077$ & $0.127 \pm 0.014$ & $0.129 \pm 0.016$ \\
Contrastive (causal) & SAT $n=50$ & cumulative & $0.002 \pm 0.000$ & $1.000 \pm 0.000$ & $0.863 \pm 0.005$ & $0.241 \pm 0.001$ & $0.088 \pm 0.000$ & $0.087 \pm 0.000$ \\
Contrastive (causal) & SAT $n=50$ & state-rebuilt & $0.002 \pm 0.000$ & $1.000 \pm 0.000$ & $0.865 \pm 0.001$ & $0.241 \pm 0.001$ & $0.088 \pm 0.000$ & $0.087 \pm 0.000$ \\
Block Dropout P090 & SAT $n=50$ & cumulative & $0.028 \pm 0.016$ & $0.722 \pm 0.145$ & $0.911 \pm 0.021$ & $0.427 \pm 0.090$ & $0.055 \pm 0.015$ & $0.070 \pm 0.011$ \\
Block Dropout P090 & SAT $n=50$ & state-rebuilt & $0.008 \pm 0.006$ & $0.996 \pm 0.008$ & $0.800 \pm 0.019$ & $0.198 \pm 0.023$ & $0.086 \pm 0.012$ & $0.088 \pm 0.004$ \\
Contrastive (SSA) & SAT $n=50$ & cumulative & $0.001 \pm 0.001$ & $1.000 \pm 0.000$ & $0.866 \pm 0.001$ & $0.242 \pm 0.001$ & $0.087 \pm 0.000$ & $0.087 \pm 0.000$ \\
Contrastive (SSA) & SAT $n=50$ & state-rebuilt & $0.001 \pm 0.001$ & $1.000 \pm 0.000$ & $0.866 \pm 0.001$ & $0.242 \pm 0.001$ & $0.087 \pm 0.000$ & $0.087 \pm 0.000$ \\
Null history & SAT $n=50$ & cumulative & $0.053 \pm 0.027$ & $0.830 \pm 0.234$ & $0.887 \pm 0.029$ & $0.301 \pm 0.071$ & $0.080 \pm 0.026$ & $0.087 \pm 0.015$ \\
Null history & SAT $n=50$ & state-rebuilt & $0.024 \pm 0.019$ & $1.000 \pm 0.000$ & $0.854 \pm 0.015$ & $0.243 \pm 0.011$ & $0.100 \pm 0.010$ & $0.094 \pm 0.007$ \\
Sliding window $k=3$ & SAT $n=50$ & cumulative & $0.089 \pm 0.010$ & $0.798 \pm 0.090$ & $0.863 \pm 0.011$ & $0.253 \pm 0.012$ & $0.107 \pm 0.015$ & $0.094 \pm 0.006$ \\
Sliding window $k=3$ & SAT $n=50$ & state-rebuilt & $0.070 \pm 0.014$ & $1.000 \pm 0.000$ & $0.843 \pm 0.007$ & $0.234 \pm 0.004$ & $0.110 \pm 0.007$ & $0.100 \pm 0.003$ \\
\bottomrule
\end{tabular}